%% file: KR20.tex
\documentclass{article}
\usepackage{kr}

\usepackage{times}
\usepackage{graphicx}

\usepackage[inference]{semantic}
\usepackage{misc}
\usepackage{xspace}
\usepackage{color}
\usepackage[dvipsnames]{xcolor}
\usepackage{pifont}
\usepackage{amsmath}
\usepackage{amssymb}

\usepackage{enumitem}
\usepackage{todonotes}

\newcommand{\nbmute}[1]{\textcolor{red}{(!)}}

\usepackage{url}
\usepackage[inference]{semantic}

\newcommand{\bet}{{\,\mathbin{\bowtie}\,}}
\newtheorem{theorem}{Theorem}
\newtheorem{example}{Example}
\newtheorem{observation}{Observation}
\newtheorem{definition}{Definition}
\newtheorem{proposition}{Proposition}
\newtheorem{lemma}{Lemma}

\title{Plausible Reasoning about $\mathcal{EL}$-Ontologies using Concept Interpolation}

\author{%
Yazm\'in  Ib\'a\~nez-Garc\'ia\and
V\'ictor Guti\'errez-Basulto\and
Steven Schockaert
\affiliations
Cardiff University,UK\\
\emails
\{ibanezgarciay, gutierrezbasultov, schockaerts1\}@cardiff.ac.uk}

\begin{document}
\maketitle

\begin{abstract}
Description logics (DLs) are standard knowledge representation languages for modelling ontologies, i.e.\ knowledge about concepts and the relations between them. Unfortunately, DL ontologies are difficult to learn from data and time-consuming to encode manually. As a result, ontologies for broad domains are almost inevitably incomplete. In recent years, several data-driven approaches have been proposed for automatically extending such ontologies. One family of methods rely on characterizations of concepts that are derived from text descriptions. While such characterizations do not capture ontological knowledge directly, they encode information about the similarity between different concepts, which can be exploited for filling in the gaps in existing ontologies. To this end, several inductive inference mechanisms have already been proposed, but these have been defined and used in a heuristic fashion. In this paper, we instead propose an inductive inference mechanism which is based on a clear model-theoretic semantics, and can thus be tightly integrated with standard deductive reasoning. We particularly focus on interpolation, a powerful commonsense reasoning mechanism which is closely related to cognitive models of category-based induction. Apart from the formalization of the underlying semantics, as our main technical contribution we provide computational complexity bounds for reasoning in \EL with this interpolation mechanism.
\end{abstract}

\section{Introduction}
In the field of AI, knowledge about concepts has traditionally been encoded using logic, often in the form of description logic ontologies \cite{baader2004description,DBLP:books/daglib/0041477}. While this approach has been highly successful in particular domains, such as health care and biomedical research, the difficulty in acquiring (description logic) ontologies has clearly hampered a more widespread adoption. In open-domain settings, it is almost impossible to exhaustively encode all relevant knowledge about the concepts of interest. As a simple example illustrating this so-called knowledge acquisition bottleneck, the SUMO ontology\footnote{\url{http://www.adampease.org/OP/}} contains the knowledge that \emph{linguine}, \emph{penne}, \emph{spaghetti}, \emph{couscous} and \emph{ziti} are types of pasta, but none of the many other types of pasta are included. 

Beyond the use of ontologies, there has also been a large interest in learning concept representations from data, such as text descriptions. Word embedding models \cite{DBLP:journals/corr/abs-1301-3781} learn such representations, for instance. Some authors have also proposed approaches that exploit semi-structured data such as WikiData, Freebase and BabelNet \cite{DBLP:conf/naacl/NeelakantanC15,camacho2016nasari,DBLP:conf/sigir/JameelBS17}. Data-driven concept representations are highly complementary to ontologies: they excel at capturing similarity but are otherwise limited in the kinds of dependencies between concepts they can capture. They are essentially tailored towards a form of \emph{inductive reasoning}: given a number of instances of some concept, they are used to predict which other entities are also likely to be instances of that concept. Conversely, traditional  ontology languages use rules to encode rigid dependencies between concepts, but they cannot capture graded notions such as similarity, vagueness and typicality. 
Description logic representations are thus rather tailored to support \emph{deductive reasoning} about concepts.


There is a growing realization that a combination of deductive and inductive reasoning about concepts is  needed in many applications \cite{DBLP:journals/jwe/HarmelenT19,DBLP:journals/semweb/dAmato20}. While several authors have started to explore ways in which such an integration can be achieved, existing work has mostly relied on heuristic methods, focusing on empirical performance rather than the underlying principles. For instance, several approaches have been proposed to exploit rules \cite{guo2016jointly,DBLP:conf/emnlp/DemeesterRR16}, and symbolic knowledge more generally \cite{xu2014rc,DBLP:conf/naacl/FaruquiDJDHS15}, to learn higher-quality vector space representations. Conversely, some authors have used vector representations to infer missing knowledge graph triples \cite{DBLP:conf/naacl/NeelakantanC15,DBLP:conf/aaai/XieLJLS16}, missing ABox assertions \cite{DBLP:conf/semweb/0001dFE13,DBLP:conf/ijcai/BouraouiS18}, or missing concept inclusions \cite{DBLP:conf/semweb/LiBS19}.

 The main focus of this paper is on the inference of plausible concepts inclusions, that is, concept inclusions which 
are not entailed from a given TBox, but which are likely to hold given the knowledge obtained from vector representations and the TBox. However, unlike in previous work,  rather than focusing on empirical performance, we aim to study the underlying principles. In particular, in  existing approaches, inductive and deductive inferences are typically decoupled. For instance, in \cite{DBLP:conf/semweb/LiBS19}, missing concept inclusions are predicted in a pre-processing step, after which the standard deductive machinery is employed. The main purpose of this paper is to propose a model-theoretic semantics in which some forms of inductive reasoning about description logic ontologies can be formalized, and which thus allows for a tighter integration between the
deductive and inductive inferences.



\smallskip To illustrate the particular setting that we consider in this paper, assume that we have the following knowledge in a given TBox about some concept $C$:
\begin{align*}
 \textit{Rabbit} \sqsubseteq C \qquad
 \textit{Giraffe} \sqsubseteq C
\end{align*}

 If we additionally have background knowledge about  rabbits, giraffes and zebras, in particular the fact that zebras satisfy all the \emph{natural} properties that rabbits and giraffes have in common (e.g.\ being mammals and herbivores), we could then make the following inductive inference, even if we know nothing else about $C$:
\begin{align}\label{eqExampleInductiveInference}
\textit{Zebra} \sqsubseteq C
\end{align}

\noindent In other words, any natural property that is known to hold for giraffes and rabbits is likely to hold for zebras as well.  In such a case, we say that zebras are \emph{conceptually between} rabbits and giraffes. The corresponding inference pattern is known as \emph{interpolation} in AI\footnote{Not to be confused with the notions of  interpolation used to relate logical theories~\cite{Craig57a,LutWo-IJCAI11}} \cite{dubois1997logical,schockaert2013interpolative} and is closely related to the notion of category based induction in cognitive science \cite{osherson1990category}. 
Importantly, vector representations of concepts, e.g.\  word embeddings, knowledge graph embeddings, could be used to obtain knowledge about conceptual betweenness.  
For instance, \cite{DBLP:journals/ai/DerracS15} found that geometric betweenness closely corresponds to conceptual betweenness in vector spaces learned with multi-dimensional scaling.



Apart from conceptual betweenness, the notion of \emph{naturalness} also plays a central role. Indeed, it is clear that the conclusion in \eqref{eqExampleInductiveInference} can only be justified by making certain assumptions on the concept $C$. If $C$ could be an arbitrary concept, the resulting inference may clearly not be valid, e.g.\ this is the case if $C=\textit{Rabbit}\sqcup \textit{Giraffe}$ (i.e.\ \textit{Rabbit} or \textit{Giraffe}). For natural properties, however,
interpolative inferences seem intuitively plausible. This idea that only some properties admit inductive inferences has been extensively studied in philosophy, among others by Goodman~\shortcite{goodman1983fact}, who called such properties \emph{projectible}. As an example of a non-projectible property, he introduced the famous example of the property \textit{grue}, which means green up to a given time point and blue afterwards. Along similar lines, Quine~\shortcite{quine} introduced the notion of ``natural kinds'' to explain why only some properties admit inductive inferences. This notion was developed by G\"ardenfors \shortcite{gardenfors2000conceptual}, who introduced the term ``natural properties'' and suggested that such properties correspond to convex regions in a suitable vector space. To determine which concepts, in a given ontology, are likely to be natural, a useful heuristic is to consider the concept name: concepts that correspond to standard natural language terms are normally assumed to be natural \cite{gardenfors2014geometry}. In this paper, we will simply assume that we are given which concept names are natural. 



In particular, we consider the following setting. We are given a standard DL ontology, in addition to a set of conceptual betweenness assertions (i.e.\ assertions of the form ``natural properties that hold for $C_1,...,C_n$ should also hold for $C$'') and a list of natural concepts. The aim is to reason about the given ontology by combining standard deductive reasoning with the aforementioned interpolation principle. Note that in this way, we maintain a clear separation between deriving knowledge from data-driven representations (i.e.\ the conceptual betweenness assertions) and the actual reasoning process. 
We particularly focus on an extension of the description logic \EL~\cite{DBLP:books/daglib/0041477}. Our motivation for choosing this logic is its simplicity. \EL allows only for conjunction and existential restrictions as constructors. Nevertheless, \EL has been successfully used for encoding large-scale ontologies like \textsc{snomed ct}. Furthermore, reasoning in \EL can be performed in polynomial time. 
Our proposed extension allows for reasoning with natural concepts and 
 conceptual betweenness, and thus supports interpolative reasoning. Formally defining the semantics of these notions requires an extension to the usual first-order semantics of description logics. Indeed, to capture e.g.\ that the concept \emph{blue} is natural while \emph{grue} is not, we cannot simply model concepts as sets of individuals. To this end, we consider two alternative approaches for characterizing natural concepts at the semantic level. First, we propose a semantics in which natural concepts are characterized using sets of features. This approach is closely related to formal concept analysis \cite{wille1982restructuring}, and is loosely inspired by the long tradition in cognitive science to model concepts in terms of features \cite{tversky1977features}. Second, we propose a semantics based on vector space representations, inspired by conceptual spaces \cite{gardenfors2000conceptual}, in which natural concepts correspond to convex regions. 

As our main technical contribution, we provide complexity bounds for concept subsumption. Concept subsumption in our considered extension of \EL is \coNP-complete under the feature-enriched semantics  and \PSpace-hard under the geometric semantics. 
The difference in complexity between the two proposed semantics intuitively stems from differences in how conceptual betweenness interacts with intersection. 






    
    
    


\section{Background}\label{sec:back}



 We briefly recall some basic notions about description logics, focusing on the \EL logic in particular.

\smallskip\noindent{\bf Syntax.}
Consider countably infinite
but disjoint sets of \emph{concept names} $\mn{N_C}$ and  \emph{role names}
$\mn{N_R}$. 
These concept and role names are combined to \emph{$\EL$ concepts}, in accordance with the following grammar, 
%
where $A \in \mn{N_C}$ and $r \in \mn{N_R}$:
$$ C,D := \top \mid A \mid C \sqcap D \mid \exists r. C $$
For instance, $A \sqcap (\exists r. (B \sqcap C))$ is an example of a well-formed $\EL$ concept, assuming $A,B,C\in \mn{N_C}$ and $r\in \mn{N_R}$.
An \emph{$\EL$ TBox (ontology) $\Tmc$} is a finite set of \emph{concept inclusions (CIs)} of the form $C \sqsubseteq D$, where $C,D$ are 
$\EL$ concepts. 

\smallskip\noindent{\bf Semantics.}
The semantics of description logics are usually given in terms of first-order interpretations $(\Delta^\Imc, \cdot^\Imc)$. Such interpretations consist of a nonempty \emph{domain} $\Delta^\Imc$ and an \emph{interpretation function} $\cdot^\Imc$, which maps 
each concept name $A$ to a subset $A^\Imc \subseteq \Delta^\Imc$ and each role name $r$ to a binary relation $r^\Imc \subseteq \Delta^\Imc \times \Delta^\Imc$.
The interpretation function  $\cdot^\Imc$ is extended to complex concepts as follows:%
\begin{align*}
(\top)^\Imc  &= \Delta^\Imc, \qquad(C \sqcap D)^\Imc = C^\Imc \cap D^\Imc,\\
(\exists r.C)^\Imc &=  \{d\in \Delta^\Imc \mid \exists d' \in C^\Imc, (d,d')\in r^\Imc\}.
\end{align*}
An interpretation \Imc \emph{satisfies} a concept inclusion $C \sqsubseteq D$ if $C^\Imc \subseteq D^\Imc$; 
it is a \emph{model} of a TBox \Tmc if it satisfies all CIs in \Tmc.
A concept    $C$ \emph{subsumes a concept $D$  relative to  a TBox \Tmc} if 
every model \Imc of \Tmc satisfies $C \sqsubseteq D$. We denote this by writing $\Tmc \models C \sqsubseteq D$.

\section{ $\EL$ with In-between and Natural Concepts } \label{secPlausibleInferenceDL}
We introduce the description logic $\EL^\bet$, which extends $\EL$ with \emph{in-between concepts} and \emph{natural concepts}.

\smallskip
\noindent \textbf{Syntax.} The main change is that we introduce the in-between constructor,  which allows us to describe the set of objects that are between two concepts. Specifically, we  write  $C \bet D $ to denote all  objects that are between the concepts $C$ and $D$. 
Further, because we will need to differentiate between concepts that are natural and concepts which are not, we will assume that $\mn{N_C}$ contains a distinguished  infinite set of \emph{natural concept names} $\mn{N^{Nat}_C}$. 
The syntax of \emph{$\EL^\bet$ concepts $C,D$} is thus defined  by the following grammar, where $A \in \mn{N_C}$, 
$A' \in \mn{N^{Nat}_C}$ and $r \in \mn{N_R}$:
\begin{align*} 
C,D &:= \top \mid A \mid C \sqcap D \mid \exists r. C \mid  N     \\
N,N'  &:= A' \mid N \sqcap N' \mid N \bet N'
\end{align*}
We will call concepts of the form $N,N'$  \emph{natural concepts}. Notably, we only allow the application of the $\bet$ constructor on natural concepts. The reason for this will become clearer once we have defined the semantics. 
%
An \emph{$\EL^\bet$ TBox} is a finite set of 
concept inclusions $C\sqsubseteq D$, where  $C,D$ are $\EL^\bet$ concepts. 

\begin{example}\label{exDefTBox}
In the following, we will consider the $\EL^\bet$ TBox $\mathcal{T}$ containing the following concept inclusions:
\begin{align}
\textit{Rabbit} &\sqsubseteq \textit{Herbivore} \label{ci:RH}\\     
\textit{Giraffe} &\sqsubseteq \textit{Herbivore} \label{ci:GH}\\ 
\textit{Zebra} & \sqsubseteq \textit{Rabbit}\bet \textit{Giraffe}\label{ci:ZRbG}\\
\textit{Herbivore} & \sqsubseteq \exists \textit{eats}.\textit{Plant}\label{ci:plants}
\end{align}
such that $\textit{Rabbit}, \textit{Zebra}, \textit{Giraffe},\textit{Herbivore} \in \mn{N^{Nat}_C}$.
\end{example}
Note that betweenness in the  proposed logic $\EL^\bet$ is modelled using a binary connective. In practice, however, the knowledge we have may relate to more concepts. Indeed, our general aim is to deal with knowledge of the form ``natural properties which hold for all of $A_1,...,A_k$ also hold for $B$'', or more precisely, that to derive $B \sqsubseteq N$ for a natural concept $N$, it is sufficient that $A_1 \sqsubseteq N,...,A_k \sqsubseteq N$ can be derived. However, in both of the semantics that we consider in this paper, the in-between operator will be associative. For $k\geq 2$, we can thus write $B \sqsubseteq A_1 \bet ... \bet A_k$ to encode such knowledge.

\smallskip
\noindent \textbf{Semantics.} Our aim is to characterize the semantics of naturalness and betweenness, in accordance with the idea of interpolation. For instance, given a TBox containing the axioms $B \sqsubseteq C \bet D$, $C \sqsubseteq N$, $D \sqsubseteq N$ with  $N$ a natural concept, we should be able to infer $B \sqsubseteq N$. This means that we need to distinguish between natural concepts and other concepts at the \emph{semantic level}, which is not possible if we simply interpret a concept as a set of objects. 
We will thus refine the usual first-order interpretations, such that we can characterize (i) which concepts are natural and (ii) which concepts are between which others. 

We will consider two possible approaches to define such semantics. First, we will consider \emph{feature-enriched semantics}, which defines a semantics in the spirit of formal concept analysis \cite{wille1982restructuring}. In this case, at the semantic level we associate a set of features with each concept. Note that these features are semantic constructs, which have no direct counterpart at the syntactic level. A concept is then natural if it is completely characterized by these features, while $B$ is between $A$ and $C$ if the set of features associated with $B$ contains the intersection of the sets associated with $A$ and $C$. Second, we will consider \emph{geometric semantics}, which follows the tradition of G\"ardenfors~\cite{gardenfors2000conceptual}. In this case, concepts will be interpreted as regions from a vector space. A concept is then natural if it is interpreted as a convex region, while $B$ is between $A$ and $C$ if the region corresponding with $B$ is geometrically between the regions corresponding with $A$ and $C$ (i.e.\ in the convex hull of their union). In the following sections, we introduce these two types of semantics in more detail.

\section{Feature-Enriched Semantics}\label{secFeatureSemantics}
In this section we introduce a refinement of the usual first-order interpretations, in which each individual is described using a set of features. Our main motivation here is to find the simplest possible semantics which is rich enough to capture betweenness and naturalness.  

 
%

\subsection{Interpretations}
The following definition introduces a refinement of the usual DL interpretations, by introducing features in the spirit of formal concept analysis (FCA). 
\begin{definition}\label{defFeatureInterpretation}
A  \emph{feature-enriched} interpretation is a tuple $\Imf= ( \Imc,  \Fmc, \pi)$, such that
\begin{enumerate}
\item $\Imc = (\Delta^\Imc, \cdot^\Imc)$ is a classical DL interpretation;
\item $\Fmc$ is a non-empty finite set of \emph{features};
\item $\pi$ is a mapping assigning to every element $d\in\Delta^\Imc$ a proper subset of features $\pi(d) \subset \Fmc$;
\item for each proper subset $\mathcal{F}' \subset \Fmc$, there exists an element $d\in\Delta^\Imc$ such that $\mathcal{F}'= \pi(d)$.
\end{enumerate}
\end{definition}
The last condition intuitively ensures that the different features are independent, by insisting that every combination of features (apart from $\mathcal{F}$ itself) is witnessed by some individual. The reason why this condition is needed relates to the fact that natural concepts will be characterized in terms of sets of features. For instance, it ensures that two concepts which are characterized by different sets of features cannot have the same extension. Note that only proper subsets of $\mathcal{F}$ are considered, such that we can associate $\mathcal{F}$ itself with the empty concept.


Under $\Imf$, a concept $C$ is interpreted as a pair $C^\Imf := \langle C^\Imc, \varphi(C)\rangle$ where $C^\Imc \subseteq \Delta^\Imc$ and $\varphi(C)$ is the set of all features from $\Fmc$ which the elements from $C^\Imc$ have in common: $$\varphi(C) := \bigcap_{d \in C^\Imc} \pi(d).$$ 
Intuitively,  we can think of $\varphi(C)$ as the set of necessary conditions that an individual needs to satisfy to belong to the concept $C$. The features from $\Fmc$ themselves can thus be seen as a set of primitive conditions that humans might rely on when categorizing individuals. However, note that the considered features do not play any role at the syntactic level, i.e.\ one cannot directly refer to them and it is not possible to specify them when encoding a TBox.

 

For standard  \EL concepts $C$, the set  $C ^\Imc$ is defined as in Section~\ref{sec:back}. For concepts of the form $N \bet N'$,
we extend the definition of $\cdot^\Imc$ as follows.
$$
(N \bet  N' )^\Imc  = \{d \in \Delta^\Imc \mid   \varphi(N)\cap \varphi(N') \subseteq \pi(d)\}.
$$
Intuitively, $(N \bet N')^\Imc$ contains all elements from the domain that have all the features that are common to both $N$ and $N'$.
Finally, each role name $r$ is interpreted as $r^\Imf := r^\Imc$.
%

\begin{definition}\label{defFeatureModel}
Let $\Tmc$ be an $\EL^\bet$ TBox and $\Imf= (\Imc, \Fmc, \pi)$ a feature-enriched interpretation. We say that  $\Imf$ is a \emph{model of $\Tmc$}, written $\Imf \models \Tmc$ if
the following hold.
\begin{enumerate}
    \item  $C^\Imc \subseteq D^\Imc$,
    for every  $C \sqsubseteq D$ in \Tmc; 
    \item for every natural concept name $A$ in $\Tmc$, it holds that
    $A^\Imc = \{d \in \Delta^\Imc \,|\, \varphi(A) \subseteq \pi(d)\}$.
%
\end{enumerate}

\end{definition}

\begin{example}\label{exInterpretation}
Let us consider the following interpretation $\Imf= ( \Imc,  \Fmc, \pi )$ of the concept names from Example \ref{exDefTBox}. We define $\Fmc = \{f_1,f_2,f_3,f_4\}$,  $\Imc = (\Delta^\Imc, \cdot^\Imc)$ and $\Delta^\Imc= \{d_1,d_2,...,d_{15}\}$. Furthermore, $\pi$ is defined as:
\begin{align*}
\pi(d_1) &{=} \{f_1\} &  \pi(d_6) &{=} \{f_1,f_3\}  & \pi(d_{11})&{=} \{f_1,f_2,f_3\}\\
\pi(d_2) &{=} \{f_2\}  & \pi(d_7) &{=} \{f_1,f_4\} & \pi(d_{12})&{=} \{f_1,f_2,f_4\}\\
\pi(d_3) &{=} \{f_3\} & \pi(d_8) &{=} \{f_2,f_3\} & \pi(d_{13})&{=} \{f_1,f_3,f_4\}\\
\pi(d_4) &{=} \{f_4\} & \pi(d_9)&{=} \{f_2,f_4\} & \pi(d_{14})&{=} \{f_2,f_3,f_4\}\\
\pi(d_5) &{=} \{f_1,f_2\} & \pi(d_{10})&{=} \{f_3,f_4\} & \pi(d_{15})&{=} \emptyset
\end{align*}
Finally, $\varphi$ and $\cdot^\Imc$ are defined as follows:
\begin{align*}
\varphi(\textit{Rabbit}) & {=} \{f_1,f_2,f_3\} & \textit{Rabbit}^\Imc &{=} \{d_{11}\}\\
\varphi(\textit{Zebra}) & {=} \{f_2,f_3\} & \textit{Zebra}^\Imc &{=} \{d_8,d_{11},d_{14} \} \\
\varphi(\textit{Giraffe}) & {=} \{f_2,f_3,f_4\}& \textit{Giraffe}^\Imc &{=} \{d_{14}\} \\
\varphi(\textit{Herbivore}) & {=} \{f_3\} &
\textit{Herbivore}^\Imc &{=} \{d_3,d_6,d_8,\\
&&&  d_{10}, d_{11}, d_{13},d_{14}\}
\end{align*}
It is easy to verify that the conditions from Definition \ref{defFeatureInterpretation} are indeed satisfied by $\Imf$.
Furthermore, we have, for instance:
$$
\varphi(\textit{Rabbit}) \cap \varphi(\textit{Giraffe}) = \{f_2,f_3\}$$
and thus 
$$
(\textit{Rabbit} \bet \textit{Giraffe})^\Imc = \{d_8,d_{11},d_{14}\} 
$$
In particular, we have that $\mathfrak{I}$ is a model of \eqref{ci:RH}--\eqref{ci:ZRbG}. Note that for simplicity, we have not considered \eqref{ci:plants} in this example (which would require considering more objects and features). 
\end{example}
The model we considered in the previous example is rather counter-intuitive, as we would normally think of \textit{Rabbit}, \textit{Zebra} and \textit{Giraffe} as disjoint concepts. However, the feature-enriched semantics is rather restrictive when it comes to modelling disjoint concepts. While this is not an issue for the logic \EL, it shows that a different semantics would be needed for extensions of \EL in which disjointness can be expressed. The geometric semantics, which we discuss in Section \ref{secGeometricSemantics}, is more general in this respect.

\subsection{Natural Concepts}
Observe that Condition 2 in Definition~\ref{defFeatureModel} enforces that the extensions of natural concept names are completely determined by their features.
This is indeed in line with the intended semantics of \emph{natural concepts} explained above. This property extends to all natural concepts.

%
%
%
%
\begin{proposition}

 Let  \Tmc be an $\EL^\bet$ TBox and let
$\mn{Nat}(\Tmc)$ denote the smallest set of concepts such that 
\begin{itemize}
\item $\top \in \mn{Nat}(\Tmc)$; 
\item every concept name $A \in \mn{N^{Nat}_C}$ occurring in \Tmc  belongs to $\mn{Nat}(\Tmc)$; 
    \item if $C,D \in \mn{Nat}(\Tmc)$, then  $C \sqcap D \in \mn{Nat}(\Tmc)$ and $C \bet D \in \mn{Nat}(\Tmc)$.
\end{itemize}
Then, for every $C \in \mn{Nat}(\Tmc)$ and every model $\Imf=(\Imc,\Fmc,\pi)$ of \Tmc
it holds that $C^\Imc = \{d \in \Delta^\Imc \mid \varphi(C) \subseteq \pi(d)\}$.
\end{proposition}
Intuitively, for a natural concept $C$, its associated set of features $\varphi(C)$ corresponds to necessary and sufficient conditions for an element to belong to the concept. A closely related property of natural concepts is that  concept inclusions can be characterized in terms of feature inclusion:
\begin{lemma}\label{prop:NatIncl}
Let $\Imf= (\Imc, \Fmc, \pi)$ be a feature-enriched interpretation and  $D$  a natural concept in $\Imf$. Then  for every concept $C$,  $\varphi(D) \subseteq \varphi(C)$ iff $C^\Imc \subseteq D^\Imc$. 
\end{lemma}

\subsection{In-between Concepts}
Feature-enriched interpretations allow us to define betweenness at the level of the objects in the domain of a given interpretation $(\Imc, \Fmc, \pi)$. For $d,d_1,d_2 \in \Delta^\Imc$, we will say that \emph{$d$ is between $d_1$ and $d_2$},  denoted by $\mn{bet}(d_1,d,d_2)$, if $\pi(d_1) \cap \pi(d_2) \subseteq \pi(d)$.
 
\begin{proposition}~\label{featureselem}
Let $\Imf= (\Imc, \Fmc, \pi)$ be a feature-enriched interpretation. 
For every pair of natural concepts $C,D$ in $\Imf$ such that $C^\Imc \neq \emptyset$ and $D^\Imc\neq \emptyset$, it holds that $(C\bet D)^\Imc$ is equal to the following set:
$$\Bbf = \{d \in \Delta^\Imc \mid \exists d_1 \in C^\Imc. \exists  d_2 \in D^\Imc \text{ s.t. }  \mn{bet}(d_1, d, d_2)\}.$$
\end{proposition}
Observe that $\Bbf$ provides an intuitive definition of betweenness and that the assumption that $C$ and $D$ are natural is crucial for showing $(C \bet D)^\Imc \subseteq \Bbf$. This justifies the syntactic restriction that $\bet$ is only applied to natural concepts.

\subsection{Link with FCA} There is a clear link between the notion of natural concept in an interpretation $(\Imc, \Fmc, \pi)$ and the notion of \textit{formal concept} from FCA. Let us consider the formal context $(\Delta^{\Imc},\Fmc, \iota)$, where the incidence relation $\iota$ is defined as $\iota(d,f)$ iff $f\in \pi(d)$, for $d\in \Delta^{\Imc}$ and $f\in \Fmc$. 
\begin{observation}
It holds that $C$ is a natural concept in $(\Imc, \Fmc, \pi)$  iff $(C^\Imc,\varphi(C))$ is a formal concept of the formal context $(\Delta^{\Imc},\Fmc, \iota)$.
\end{observation}
Indeed the following two conditions are satisfied:
\begin{align}
C^\Imc &= \{d \in \Delta^\Imc \mid \iota(d,f) \text{ for all $f\in \varphi(C)$}\} \label{eqFormalConceptCondition1}\\
\varphi(C) &= \{f \in \Fmc \mid \iota(d,f) \text{ for all $d\in C^\Imc$}\}\label{eqFormalConceptCondition2}
\end{align}
Condition \eqref{eqFormalConceptCondition1}  follows from the definition of natural concept, while  \eqref{eqFormalConceptCondition2} follows from the definition of $\varphi(C)$.

Furthermore, note that Conditions 3 and  4 in  Definition~\ref{defFeatureInterpretation} ensure that $(\emptyset, \Fmc)$ is also a formal concept, in fact the least element of the concept lattice. In other words, if $C$ is a natural concept in $\Imf$ such that $C^\Imc = \emptyset$, then $\varphi(C) = \Fmc$. One consequence of this property is the following. 
\begin{observation}\label{obs:NatBet} For every interpretation $(\Imc, \Fmc, \pi)$, it holds that $(C \bet D)^\Imc \neq \emptyset$ iff $C^\Imc \neq \emptyset$ or $D^\Imc \neq \emptyset$. 
\end{observation}

\subsection{Interpolation in Feature-enriched Models}
The following example illustrates how the feature-enriched semantics  enables interpolative inferences.
\begin{example}
Consider again the TBox $\mathcal{T}$ from Example \ref{exDefTBox} and the interpretation $\Imf$ from Example \ref{exInterpretation}. It is easy to verify that $\Imf$ is a model of $\mathcal{T}$. Moreover, $\Imf$ also satisfies the following concept inclusion:
$$
\textit{Zebra} \sqsubseteq \textit{Herbivore},
$$
which follows the inference pattern explained in the introduction, given that $\textit{Herbivore}$ is a natural concept.
In fact, this concept inclusion is entailed by $\mathcal{T}$. To see this, note that in any model $(\Jmc, \Fmc, \pi)$ of $\Tmc$, because of the concept inclusions \eqref{ci:RH} and \eqref{ci:GH},
it holds that 
\begin{align*}
    \varphi(\textit{Herbivore}) &\subseteq \varphi(\textit{Rabbit}) & \varphi(\textit{Herbivore}) &\subseteq \textit{Giraffe}.
\end{align*}    
That is, $$\varphi(\textit{Herbivore}) \subseteq  \varphi(\textit{Rabbit}\bet \textit{Giraffe}).$$ 
By Lemma~\ref{prop:NatIncl}, and because of $\textit{Herbivore} \in \mn{N^{Nat}_C}$, we have that $(\textit{Rabbit}\bet \textit{Giraffe})^\Jmc$ $\subseteq$ $\textit{Herbivore}^\Jmc$. Finally by concept inclusion \eqref{ci:ZRbG} we can conclude $\textit{Zebra}^\Jmc \subseteq \textit{Herbivore}^\Jmc$. 
\end{example}
Clearly, the arguments used in the above example generalize. We thus have the following result, which provides the soundness of \emph{interpolative inferences}. 
\begin{lemma} \label{lemma:int}
Let $\Tmc$ be an $\EL^\bet$ TBox, and $C,D,B$ be natural concepts w.r.t\ \Tmc. 
If $\Tmc \models \{C \sqsubseteq B, D\sqsubseteq B\}$ then 
$\Tmc \models C \bet D \sqsubseteq B$. 
\end{lemma}
However, the applicability of this lemma is limited, as it requires  specific knowledge about $C \bet D$. For example, consider the TBox $\Tmc'$ containing the following assertions: 
\begin{align*}
 A \sqcap C &\sqsubseteq B  &
& A \sqcap D \sqsubseteq B & 
X &\sqsubseteq C\bet D &
\end{align*}
with $B$ a natural concept w.r.t\ $\Tmc'$.
Since $C \bet D$ is characterised by all common features of $C$ and $D$,  a plausible inference from $\Tmc'$ is that  
$ A \sqcap (C \bet D) \sqsubseteq B$ holds, which in turns allows us to draw the conclusion that $A \sqcap X \sqsubseteq B$. However, using Lemma~\ref{lemma:int} 
we can only soundly infer that $\Tmc' \models (A \sqcap C) \bet (A \sqcap D) \sqsubseteq B$, provided that $A\sqcap C$ and $A \sqcap D$ are both natural w.r.t.\ $\Tmc'$.
Thus, we shall investigate under which conditions $A \sqcap (C \bet D) \sqsubseteq (A\sqcap C) \bet (A \sqcap D)$ holds.

\begin{lemma}\label{prop:ConjFeat}
Let $C,D$ be natural concepts w.r.t.\ a given TBox \Tmc. For every model $\Imf$ of $\Tmc$, it holds that $\varphi(C \sqcap D) = \varphi(C) \cup \varphi(D)$. 
\end{lemma}

We use this property to show that interpolation pattern exemplified above is indeed sound for natural concepts. 
%
\begin{theorem}\label{lemma:intcj}
Let \Tmc be an $\EL^\bet$-TBox, 
and let $A,B,C,D$ be natural concepts w.r.t.\ \Tmc.
If $\Tmc \models \{ A \sqcap C \sqsubseteq B,  A\sqcap D \sqsubseteq B \}$
then $\Tmc \models  A \sqcap  (C \bet D)  \sqsubseteq B$.
\end{theorem}
\section{Geometric Semantics}\label{secGeometricSemantics}
We now turn to a different approach for defining the semantics of $\bet$ and natural concepts, which is inspired by conceptual spaces~\cite{gardenfors2000conceptual}. The main idea is that concepts are represented as regions in a Euclidean space, with natural concepts corresponding to convex regions. One important advantage of the geometric semantics is that it is closer to the vector space embeddings that are commonly used when learning concept representations from data. In other words, if knowledge about conceptual betweenness is learned from vector space representations, then it seems natural to define the semantics in a similar way. Another advantage is that the geometric semantics avoids some of the counter-intuitive restrictions of the feature-enriched semantics, in terms of how betweenness and disjointness interact. This means that the geometric semantics can also be used for extensions of \EL in which disjointness can be expressed, although we leave a detailed study of the computational properties of interpolation in such extensions as a topic for future work. On the other hand, as we will see in the next section, these advantages come at a computational cost, even when staying within the context of \EL.

One key issue of the geometric semantics is that, unlike for the feature-enriched semantics, $X \sqsubseteq C \bet D$ does not imply $X \sqcap A \sqsubseteq (C \sqcap A) \bet (D\sqcap A)$, even when all of the concepts involved are natural. For this reason, we extend the language of $\EL^{\bet}$ with assertions of the form $A \lessdot (C,D)$, where $A$ is a natural concept and $C,D$ are natural concept names. 
We will refer to these expressions as \emph{non-interference assertions}. Their aim is to encode how  $C\bowtie D$ interacts with intersections with $A$ (explained in more detailed below). We will refer to the resulting logic as $\EL^\bet_{\textit{reg}}$. In particular, \emph{$\EL^\bet_{\textit{reg}}$ TBoxes} are finite sets of concept inclusions and non-interference assertions.

 


\subsection{Interpretations}
Geometric interpretations represent concepts as regions, where individuals are intuitively represented as points. In addition to specifying these regions, however, geometric interpretations also specify some additional mappings, which will be needed to formalize the idea of non-interference.

In what follows, we use  
$\mn{conv}(X)$ to denote  the \emph{convex hull} of $X$, that is the intersection of all the convex sets that contain $X$, and $\oplus$ to denote the concatenation of vectors.

\begin{definition}[Geometric interpretation]
Let $\Sigma \subset \mn{N_C} \cup \mn{N_R}$.
An \emph{$m$-dimensional geometric $\Sigma$-interpretation $\Imc$} 
assigns to 
every concept name $A \in \Sigma$ a region $\mn{reg}_\Imc(A) \subseteq \mathbb{R}^m$ and 
to every role $r\in \Sigma$ a region $\mn{reg}_\Imc(r) \subseteq \mathbb{R}^{2\cdot m}$. Furthermore, $\mathcal{I}$ specifies for all natural concept names $A,B$ in $\Sigma \cap \mn{N^{Nat}_C}$, a mapping $\kappa_{(A,B)}^{\mathcal{I}}$ from $\mn{conv}(\text{reg}_{\mathcal{I}} (A) \cup \text{reg}_{\mathcal{I}} (B))$ to $\text{reg}_{\mathcal{I}} (A) \times \text{reg}_{\mathcal{I}} (B)$ such that for  every $p \in \mn{conv}(\text{reg}_{\mathcal{I}} (A) \cup \text{reg}_{\mathcal{I}} (B))$ with $\kappa_{(A,B)}^{\mathcal{I}}(p)=(p_1,p_2)$ it holds that
\begin{itemize}
\item  $p$ is between $p_1$ and $p_2$, i.e.\ $p= \lambda p_1 + (1-\lambda)p_2$ for some $\lambda\in[0,1]$;
\item $\kappa_{(B,A)}^{\mathcal{I}}(p)=(p_2,p_1)$.
\end{itemize}
\end{definition}
%
%
Note that for a point $p\in \mn{conv}(\text{reg}_{\mathcal{I}} (A) \cup \text{reg}_{\mathcal{I}} (B))$ it is always possible to find points $p_1 \in \text{reg}_{\mathcal{I}} (A)$ and $p_2 \in \text{reg}_{\mathcal{I}} (B)$ such that $p$ is between $p_1$ and $p_2$. Intuitively, however, the mapping $\kappa_{(A,B)}$ selects the pair $(p_1,p_2)$ which is most ``similar'' to $p$. This intuition will be made explicit when discussing the semantics of non-interference assertions below.

The interpretation of complex $\EL^\bet$ concepts is defined as follows.
\begin{align*}
   & \mn{reg}_\Imc(\top)  = \mathbb{R} \\
   & \mn{reg}_\Imc(C \sqcap D) = \mn{reg}_\Imc(C) \cap \mn{reg}_\Imc(D)\\
   & \mn{reg}_\Imc(\exists r. C)  = \{ p \in \mathbb{R}^m \mid \exists p' \in \mn{reg}_\Imc(C),  p {\oplus} p' \in \mn{reg}_\Imc(r) \}\\
   & \mn{reg}_{\Imc}( C_1 \bet C_2)  = \mn{conv}(\mn{reg}_\Imc(C_1) \cup \mn{reg}_\Imc(C_2))
   \end{align*}
%
Note how the definition of $\mn{reg}_{\Imc}( C_1 \bet C_2)$ defines conceptual betweenness in terms of geometric betweenness, i.e.\ the instances of the concept $C_1 \bet C_2$ are intuitively those individuals which are modelled by points that are geometrically between the regions modelling $C_1$ and $C_2$.

\begin{figure}
\centering
    \includegraphics[width=150pt]{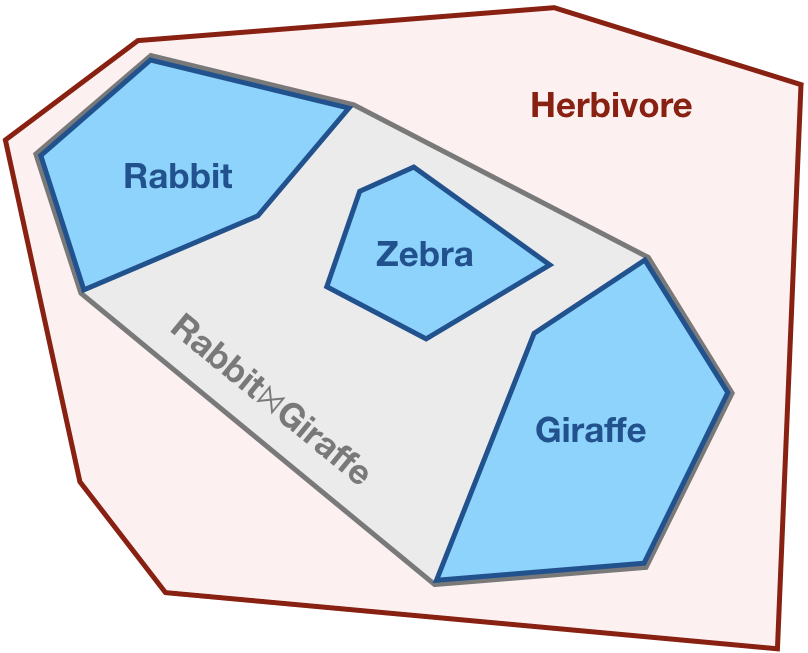}
    \caption{Illustration of a two-dimensional geometric $\Sigma$-interpretation.}
    \label{figGeometricEx}
\end{figure}

\begin{example}
Figure \ref{figGeometricEx} depicts a two-dimensional $\Sigma$-interpretation of the concepts \textit{Rabbit}, \textit{Zebra}, \textit{Giraffe} and \textit{Herbivore} from Example \ref{exDefTBox}.
\end{example}
The semantics of $\EL^\bet_{\textit{reg}}$ TBox assertions is defined as follows.
An m-dimensional  $\Sigma$-interpretation $\Imc$ \emph{satisfies a concept inclusion $C \sqsubseteq D$}, for $C,D \in \Sigma$, if 
$\mn{reg}_\Imc(C) \subseteq \mn{reg}_\Imc(D)$. 
%
The interpretation $\mathcal{I}$ \emph{satisfies the non-interference assertion $X\lessdot(A,B)$} if for all $p\in (X  \sqcap (A \bet B))^\Imc$,  whenever $\kappa_{(A,B)}^{\mathcal{I}}(p)=(p_1,p_2)$,  it holds that $p_1 \in X^\Imc$. 


The intuition behind non-interference relies on the notion of domains from the theory of conceptual spaces. For instance, if $X \sqsubseteq \textit{Red} \bowtie \textit{Blue}$ then we would expect that $(X \sqcap \textit{Small}) \sqsubseteq (\textit{Red} \sqcap \textit{Small}) \bowtie (\textit{Blue} \sqcap \textit{Small})$ also holds. This is because $\textit{Red}$ and $\textit{Blue}$ are  defined in the color domain, whereas \textit{Small} is defined in the size domain. Concepts that rely on disjoint sets of domains intuitively cannot interfere with betweenness assertions. In our setting, we only have a single vector space, whereas in the theory of conceptual spaces each domain corresponds to a separate vector space. Instead, as was argued in \cite{DBLP:conf/ecai/JameelS16}, we can think of such domains as sub-spaces of $\mathbb{R}^m$. The intended intuition is that $\kappa_{(A,B)}^{\mathcal{I}}(p)$ selects a pair of points $(p_1,p_2)$ which only differ from $p$ in the sub-spaces of the domains that are relevant to $A$ and $B$. The statement $X\lessdot(A,B)$ then intuitively asserts that the domains that are relevant for $A$ and $B$ are disjoint from the domains that are relevant for $X$.

\begin{figure}
    \centering
    \includegraphics[width=100pt]{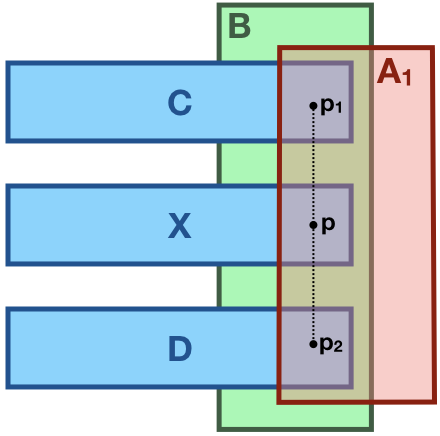}
    \hfill
    \includegraphics[width=100pt]{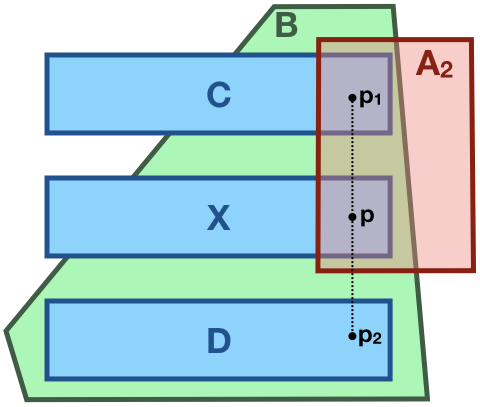}
    \caption{A configuration in which both $A_1 \lessdot (C,D)$ and $A_1 \lessdot (D,C)$ can be satisfied  (left), and a configuration in which $A_2 \lessdot (C,D)$ can be satisfied but not $A_2 \lessdot (D,C)$ (right)}.
    \label{figNonInterference}
\end{figure}

\begin{example}
The left-hand side of Figure \ref{figNonInterference} depicts a configuration in which the non-interference assertions $A_1 \lessdot (C,D)$ and $A_1 \lessdot (D,C)$ can be satisfied. In particular, a suitable mapping $\kappa_{(A,B)}(p)=(p_1,p_2)$ can be found by choosing $p_1 \in \text{reg}_{\mathcal{I}}(C)$ and $p_2 \in \text{reg}_{\mathcal{I}}(D)$ such that $p_1$ and $p_2$ share their first coordinate with $p$. Similarly, the right-hand side of Figure \ref{figNonInterference} shows a configuration in which $A_2 \lessdot (C,D)$ can be satisfied, but not $A_2 \lessdot (D,C)$.
\end{example}

\begin{definition}
Let $\Tmc$ be an $\EL^{\bet}_{\textit{reg}}$ TBox.  
An $m$-dimensional geometric $\Sigma$-interpretation $\Imc$ is an \emph{$m$-dimensional $\Sigma$-model of  \Tmc} if the following are satisfied: 
\begin{enumerate}
    \item all the concept and role names appearing in $\Tmc$ are included in $\Sigma$;
    \item \Imc satisfies every concept inclusion in $\Tmc$;
\item \Imc satisfies every 
non-interference assertion in $\Tmc$; 
\item  $\mn{reg}_\Imc(A)$ is a convex region for every concept name $A \in \Sigma \cap \mn{N^{Nat}_C}$.
\end{enumerate}
We refer to $\Imc$ as a geometric model, or simply a model, if $m$ and $\Sigma$ are clear from the context.
\end{definition}
%
%
\subsection{Interpolation in Geometric Models}
\begin{figure}
    \centering
    \includegraphics[width=100pt]{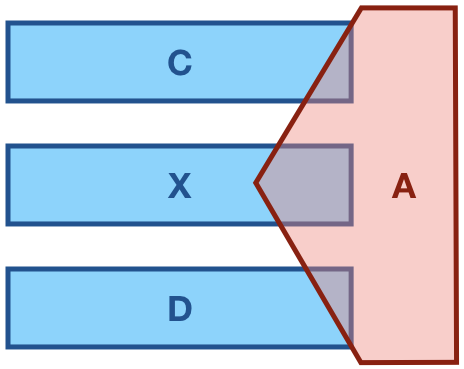}
    \hfill
    \includegraphics[width=100pt]{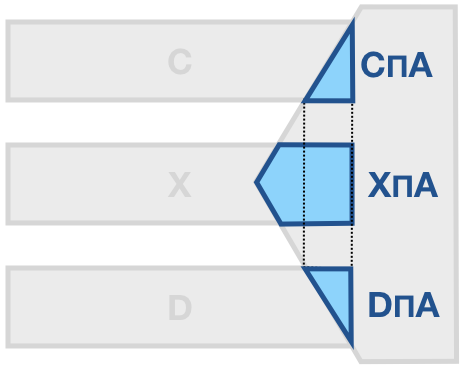}
    \caption{Even though the region representing $X$ is between those representing $C$ and $D$ (left), the region for $X\sqcap A$ is not between those for $C\sqcap A$ and $D\sqcap A$ (right).}
    \label{figInterference}
\end{figure}
It is easy to see that Lemma \ref{lemma:int} also holds for the geometric semantics, showing that basic interpolative inferences are sound. The reason why we need to consider non-interference is related to the interaction between $\bet$ and $\sqcap$. Recall from Section \ref{secFeatureSemantics} that interpolation with inclusions of the form  $A \sqcap C \sqsubseteq B$ and $A \sqcap D \sqsubseteq B$ required   $A, C $ and  $D$ to be natural concepts in order to have e.g.,  that $\varphi(A \sqcap C) = \varphi(A) \cup \varphi(C)$. 
 In the case of the geometric semantics, naturalness alone is not sufficient to allow us to derive $\Tmc \models A \sqcap (C\bet D) \sqsubseteq B$ from $\Tmc \models \{ A \sqcap C \sqsubseteq B,  A\sqcap D \sqsubseteq B \}$. To see this, consider the $2$-dimensional interpretation illustrated in Figure \ref{figInterference}.  We have that $\mn{reg}_\Imc(X)$ is between $\mn{reg}_\Imc(C)$ and $\mn{reg}_\Imc(D)$ and that $\mn{reg}_\Imc(C\sqcap A)$ and $\mn{reg}_\Imc(D\sqcap A)$ are convex (which is the geometric characterization of naturalness). Nonetheless, we can see that $\mn{reg}_\Imc(X\sqcap A)$ is not between $\mn{reg}_\Imc(C\sqcap A)$ and $\mn{reg}_\Imc(D\sqcap A)$. 
A way to enable interpolative reasoning between conjunctions of concepts is to require that the concepts involved are non-interfering. This is formalized in the following lemma.

\begin{proposition}\label{propSoundnessGeoLeftNI}
Let \Tmc be an $\EL^\bet_{\textit{reg}}$-TBox.
If $\mathcal{T}\models \{A\lessdot(C,D), A \sqcap C \sqsubseteq B, D \sqsubseteq B\}$, with $A,B$ natural concepts and $C,D \in \mn{N^{Nat}_C}$, then $\mathcal{T} \models A\sqcap (C \bowtie D) \sqsubseteq B$. 
\end{proposition}

\noindent Analogously, we can show the following result.

\begin{proposition}\label{propSoundnessGeoBoxNI}
Let \Tmc be an $\EL^\bet_{\textit{reg}}$-TBox.
If $\mathcal{T}\models \{A\lessdot(C,D), A\lessdot(D,C), A \sqcap C \sqsubseteq B,A \sqcap  D \sqsubseteq B\}$, with $A,B$ natural concepts and $C,D \in \mn{N^{Nat}_C}$, then $\mathcal{T} \models A\sqcap (C\bowtie D) \sqsubseteq B$. 
\end{proposition}

\begin{example}
The configurations in Figure \ref{figNonInterference} illustrate how sound interpolative inferences can be made when the conditions from Proposition \ref{propSoundnessGeoLeftNI} (right-hand side) or Proposition \ref{propSoundnessGeoBoxNI} (left-hand side) are satisfied.
\end{example}

\noindent Finally, the following result shows that non-interference is closed under intersection.
\begin{proposition}\label{prop:ConjGeo}
Let $A,B$ be natural concepts and $C,D$ natural concept names.
If $\mathcal{T} \models A\lessdot (C,D)$ and $\mathcal{T} \models B\lessdot (C,D)$ then $\mathcal{T} \models (A\sqcap B)\lessdot (C,D)$.
\end{proposition}


\section{Complexity of Reasoning with Interpolation}~\label{sec:complex}
We next analyze  the computational complexity of reasoning in $\EL^\bet$ and $\EL^\bet_{\! \textit{reg}}$. We  show that the ability to perform interpolation increases the  complexity of concept subsumption relative to a TBox. 
\subsection{Concept Subsumption in  $\EL^\bet$}~\label{sec:feacompl}
 We start by studying  $\EL^\bet$ and establish that concept subsumption relative to $\EL^\bet$-TBoxes is \coNP-complete.  
 We show hardness by reducing non-entailment in propositional logic to concept subsumption in $\EL^\bet$. The main underlying idea is that a  concept inclusion of the form:
 $$
 X_1 \sqcap ... \sqcap X_n \sqsubseteq Y_1 \bowtie ...\bowtie Y_m
 $$
can be used to simulate a propositional clause of the following form:
 $$
 \neg y_1 \vee ... \vee \neg y_n \vee x_1 \vee ... \vee x_n
 $$
 where each atom $x_i$ or $y_i$ is associated with a natural concept name $X_i$ or $Y_i$. This correspondence allows us to reduce the problem of entailment checking in propositional logic to the problem of checking concept subsumption relative to $\EL^\bet$-TBoxes; the proof can be found in the online appendix.
 \begin{theorem}\label{thm:upperfeatures}
Concept subsumption relative to an $\EL^\bet$-TBox is \coNP-hard, even when restricting to TBoxes without any occurrences of existential restrictions.
 \end{theorem}

\noindent For the matching upper bound we provide a polynomial 
time guess-and-check procedure. We assume that $\EL^\bet$-TBoxes are in the following normal form. For a TBox to be in normal form, we require that every concept inclusion is of one of the forms $A\sqsubseteq B$, $A_1 \sqcap A_2  \sqsubseteq  B,$     $A \sqsubseteq \exists r. B,$     $\exists r. A \sqsubseteq B,$  $A \sqsubseteq  B_1 \bet B_2,$   $B_1 \bet B_2  \sqsubseteq A,$ where $A,A_1, A_2, B$ are concept names or the concept $\top$  and  $B_1,B_2$ are natural concept names. It is standard to show that every TBox can be transformed into this normal form in polynomial time such that (non-)subsumption between the concept names that occur in the original TBox is preserved.
%

We start by showing the following property of feature-enriched interpretations. 

\begin{lemma}\label{lema:boundedModel}
Let \Tmc be an $\EL^\bet$ TBox. For every model $\Imf = (\Imc, \Fmc, \pi)$ of \Tmc, there is a model 
$\widehat \Imf = (\Imc, \widehat \Fmc, \hat \pi)$ such that $|\widehat \Fmc | \leq \mn{poly}(\Tmc)$.  
\end{lemma}

Before presenting the decision procedure, we introduce some notions.
Let \Tmc  be an  $\EL^\bet$ TBox. A \emph{feature assignment for  \Tmc} from a set of features $\mathcal F$    is a mapping $\theta$ assigning to each concept name in \Tmc a subset $F \subseteq \mathcal F$. We say that a feature assignment $\theta$ for \Tmc is \emph{proper} if  the following conditions hold:
 
 \begin{enumerate}
 \item For every concept inclusion of the form $A_1 \sqcap A_2 \sqsubseteq B$ in $\Tmc$ with $A_1,A_2$ natural concept names,  it holds that  
$\widehat\theta(B) \subseteq \theta(A_1) \cup \theta(A_2)$;
\item for every concept inclusion of the form   $A  \sqsubseteq C$ in $\Tmc$,  it holds that $\widehat \theta(C) \subseteq \theta(A)$, 
\end{enumerate}
with $\widehat \theta(\cdot)$ defined as follows: $\widehat\theta(A) = \theta(A)$;
 \begin{align*}
 &  \widehat\theta(A \bet B) = \theta(A) \cap \theta(B);  &&\widehat\theta(\top) = \widehat\theta(\exists r. B) = \emptyset. 
\end{align*}

We are now ready to describe our guess-and-check procedure
 to decide non-subsumption in $\EL^\bet$. 
Given an $\EL^\bet$  TBox $\Tmc$,  
we proceed as follows:  
\begin{enumerate}

 \item Guess a feature assignment $\theta$ for \Tmc from some set of features \Fmc (By Lemma~\ref{lema:boundedModel}, we can assume that $|\Fmc| \leq \mn{poly}(\Tmc)$).
\item Add  concept inclusions $A \sqsubseteq C$ to $\Tmc$   if $\widehat \theta(C) \subseteq \theta(A)$, for $A\in \mn{N_C}$  and $C \in \mn{N_C^{Nat}}$ or  $C = B_1 \bet B_2$, occurring in $\Tmc$. Let   $\Tmc_{\theta}$  be the TBox obtained after this step. 
\item  Compute the completion $\Tmc'_\theta$ of $\Tmc_\theta$ using the classical completion algorithm for \EL~\cite{DBLP:books/daglib/0041477}, where concepts of the form $A \bet B$ are regarded as concept names. Let $\Tmc'_\theta$ be the TBox obtained after this step. 
\item  Check that $\theta$ is proper for  $\Tmc'_\theta$.
  \end{enumerate}
 \begin{lemma}\label{lemma:coNPUpper}
 Let \Tmc be an $\EL^\bet$ TBox and $A,B$ concept names. Then, $\Tmc \not \models A \sqsubseteq B$ iff, after applying Steps 1-4 above, 
 $A \sqsubseteq B \not \in \Tmc'_\theta$. 
 \end{lemma}
Summing up we obtain the following result.
\begin{theorem}~\label{thm:lowerfeatures}
Concept subsumption relative to  $\EL^\bet$-TBoxes is \coNP-complete.
\end{theorem}
%
%
\subsection{Concept Subsumption in $\EL^\bet_{\textit{reg}}$ \label{secHardnessGeometric}} 
We move now to investigate concept subsumption relative
to $\EL^\bet_{\textit{reg}}$ TBoxes (under the geometric semantics). In this case, we are able to show the following hardness result.

%
\begin{theorem}~\label{thm:lowergeo}
Concept subsumption relative to  $\EL^\bet_{\textit{reg}}$-TBoxes is \PSpace-hard.
\end{theorem}
Inspired by \cite{schockaert2013interpolative}, the proof proceeds by reduction from the dominance problem in generalized CP-nets (GCP-nets). We briefly sketch this reduction. First we recall the dominance problem. A GCP-net over  a set of propositional atoms $\mathcal{A}=\{a_1,...,\allowbreak a_m\}$ is specified by a set of so-called \emph{conditional preference (CP) rules} $\rho_i$ ($i\in \{1,...,n\}$) of the following form:
\begin{align}\label{eqCPrule}
\eta_i: q_i > \neg q_i
\end{align}
where $\eta_i$ is a conjunction of literals and $q_i$ is a literal (over $\mathcal{A}$). The intuition of this rule is that whenever $\eta_i$ is true, then it is better to have $q_i$ true than to have $q_i$ false, assuming that everything else stays the same (i.e.\ \textit{ceteris paribus}). An \emph{outcome} is defined as a tuple of literals $(l_1,...,l_m)$ where $l_i$ is either $a_i$ or $\neg a_i$. Outcomes thus encode possible worlds. Let $\omega_1 = (l_1,...,\neg q_i,...,l_m)$ and $\omega_2=(l_1,...,q_i,...,l_m)$ be outcomes which only differ in the truth value they assign to $q_i$ and which both satisfy the condition $\eta_i$. Then we say that the rule $\rho_i$ \emph{sanctions an improving flip} from $\omega_1$ to $\omega_2$. Moreover, we say that an outcome $\omega$ \emph{dominates} an outcome $\omega'$, written $\omega' \prec \omega$ if there exists a sequence of improving flips from $\omega'$ to $\omega$. A GCP-net is \emph{consistent} if there are no cycles of improving flips, i.e.\ there are no outcomes $\omega$ for which $\omega\prec \omega$. It was shown in \cite{DBLP:journals/jair/GoldsmithLTW08} that the problem of checking whether some outcome $\omega$ dominates an outcome $\omega'$ is \PSpace-complete, even when restricted to consistent GCP-nets.

The proposed reduction is defined as follows. Let the initial outcome be given by $(l_1,...,l_m)$.
The corresponding $\EL^\bet_{\textit{reg}}$ TBox $\mathcal{T}$ contains the following corresponding concept inclusion:
\begin{align} \label{eqInitialELrule}
\tau(l_1) \sqcap ... \sqcap \tau(l_m) \sqsubseteq Z
\end{align}
where the mapping $\tau$ is defined by $\tau(a_i) = A_i$ and $\tau(\neg a_i) = \overline{A_i}$, with $A_i$, $\overline{A_i}$ and $Z$ natural concept names. We furthermore extend this mapping to conjunctions of literals as $\tau(l_1 \wedge ... \wedge l_k) = \tau(l_1) \sqcap ... \sqcap \tau(l_k)$. For each CP-rule $\rho_i$, we have that $\mathcal{T}$ contains the following concept inclusions:
\begin{align}
 X_i &\sqsubseteq Z \label{eqELrule2}\\
 \tau(\eta_i \wedge q_i) &\sqsubseteq W_i\bowtie X_i\label{eqELrule3}\\
 W_i &\sqsubseteq \tau(\eta_i  \wedge \neg q_i) \label{eqELruleWi1}\\
 \tau(\eta_i \wedge \neg q_i) &\sqsubseteq W_i\label{eqELruleWi2}
\end{align}
where $W_i$ and $X_i$ are natural concept names. Furthermore, for each rule $\eta_i: q_i > \neg q_i$ and each $a_j$, such that neither $a_j$ nor $\neg a_j$ occurs in $\eta_i$ or $\{q_i, \neg q_i\}$, we add the non-interference assertions $A_j \lessdot (W_i,X_i)$ and $\overline{A_j} \lessdot (W_i,X_i)$.

Then we can show that $(r_1,...,r_m)$ dominates $(l_1,...,l_m)$, with $(r_1,...,r_m)\neq (l_1,...,l_m)$ iff
$$
\mathcal{T} \models \tau(r_1 \wedge  ... \wedge r_m) \sqsubseteq Z
$$
We conjecture that a matching \PSpace upper bound can be found, although this currently remains an open question.



%
\section{Related Work}
The problem of automated knowledge base completion has received significant attention in recent years. Most of the work in this area has focused on completing knowledge graphs by learning a suitable vector space representation of the entities and relations involved \cite{NIPS20135071}. However, some work has also focused on description logic ontologies. An early example is \cite{DBLP:conf/semweb/dAmatoFFGL09}, which proposed to use a similarity metric between individuals to find plausible answers to queries. More recently, Bouraoui and Schockaert~\shortcite{DBLP:conf/ijcai/BouraouiS18} proposed a method for finding plausible missing ABox assertions, by representing each concept as a Gaussian distribution in a vector space, while Kulmanov et al.~\shortcite{DBLP:conf/ijcai/KulmanovLYH19} proposed a method to learn a vector space embedding of $\mathcal{EL}$ ontologies for this purpose. The problem of completing TBoxes using vector space representations was considered in \cite{DBLP:conf/aaai/BouraouiS19}.

The previously mentioned approaches are essentially heuristic, focusing on the empirical performance of the considered strategies, without introducing a corresponding model-theoretic semantics or studying the formal properties of associated reasoning tasks (e.g.\ computational complexity). The problem of formally combining logics and similarity is addressed in \cite{CSL2005,SheremetWZ10}, where an operator is introduced to express that a concept $A$ is more similar to some concept $B$ than to some concept $C$. By focusing on comparative similarity, the problem of dealing with numerical degrees is avoided. We can thus think of comparative similarity and conceptual betweenness as two complementary approaches for reasoning about similarity in a qualitative way. 
%
Related to concept  betweenness is the notion of \emph{least common subsumer (LCS)} which has been broadly studied in the context of DLs  as means for supporting inductive inference~\cite{CohenBH92,KustersB01,BaaderST07,EckeT12,ZarriessT13,JLW-AAAI20}. Similarly to the LCS of $A$ and $B$, $lcs(A,B)$,  $A \bet B$ subsumes both $A$ and $B$,  thus generalizing them. However, $lcs(A,B)$ is minimal w.r.t.\ the extensions of $A$ and $B$, whereas $A \bet B$ is minimal w.r.t.\ their intent under the feature-based semantics, i.e., it is the least common `natural subsumer'. The latter is arguably, closer to to the cognitive notion of the least common subsumer of $A$ and $B$ as the concept capturing their commonalities. Further, in \EL, where a syntactic description of LCS is not guaranteed to exist, betweeness provides such description.  
%
Beyond qualitative approaches, it is also possible to directly model degrees of similarity. For instance, Esteva et al.~\shortcite{esteva1997modal} considered a graded modal logic which formalizes a form of similarity based reasoning. Fuzzy and rough description logics can also be viewed from this angle \cite{straccia2001reasoning,BCE+-15,SchlobachKP07,KleinMS07,JiangTWT09,PenalozaZ13,LisiS13}. Within a broader context, \cite{DBLP:conf/ismis/LietoP18} is also motivated by the idea of combining description logics with ideas from cognitive science, although their focus is on modelling typicality effects and compositionality, e.g.\ inferring the meaning of \emph{pet fish} from the meanings of \emph{pet} and \emph{fish}, which is a well-known challenge for cognitive systems since typical pet fish are neither typical pets nor typical fish.

The idea of providing a semantics for description logics in which concepts correspond to convex regions from a conceptual space was already considered in \cite{OzcepM13}.  Guti\'errez-Basulto and Schockaert~\shortcite{DBLP:conf/kr/Gutierrez-Basulto18} also studied a semantics based on conceptual spaces for existential rules. 
The idea of linking description logic concepts to feature based models has been previously considered as well. For instance, Porello et al.~\shortcite{DBLP:conf/dlog/PorelloKRTGM19} introduced a syntactic construct to define description logic concepts in terms of weighted combinations of properties, although unlike in this paper, their properties/features are also syntactic objects. Description logics with \emph{concrete domains} provide  means to refer to concrete objects, such as numbers or spatial regions~\cite{Lutz02,LutzM07}, but (unlike in our case) they come equipped with syntax to access and impose constraints on these domains.  

The study of the link between DLs and formal concept analysis has received considerable attention, see e.g.~\cite{BaaderM00,Rudolph2006,BaaderGSS07,Sertkaya10,Distel2011} and references therein. However, unlike in these works, our main objective in this paper is to use features to characterize natural concepts, and provide semantics capturing the idea of interpolation.

\section{Conclusions and Future Work}
Our central aim in this paper was to formalize interpolative reasoning, a commonsense inference pattern that underpins a cognitively plausible model of induction, in the context of description logics. To this end, we have studied extensions of the description logic $\EL$ in which we can encode that one concept is between two other concepts.  In particular, we have studied two approaches to formally define the semantics of betweenness and the related notion of naturalness: one inspired by formal concept analysis and one inspired by conceptual spaces. We furthermore showed that reasoning in the considered extensions of $\EL$ is \coNP-complete under the featured-enriched semantics, and \PSpace-hard under the geometric semantics. 

There are several important avenues for future work. First, at the foundational level, we believe that our framework can be used as a basis for integrating inductive and deductive reasoning more broadly. Essentially, inductive reasoning requires two things: (i) we need knowledge about the representation of concepts in a suitable feature space and (ii) we need to make particular restrictions on how concepts are represented in that feature space. In this paper, (i) was addressed by providing knowledge about conceptual betweenness whereas (ii) was addressed by the notion of naturalness. However, there may be several other mechanisms to encode knowledge about the feature space. One possibility is to have assertions that relate to analogical proportions (i.e.\ assertions of the form ``$a$ is to $b$ what $c$ is to $d$''), which can be formalized in terms of discrete features or geometric representations \cite{miclet2008analogical,DBLP:series/sci/PradeR14a}. 

Another important line for future work relates to applying the proposed framework in practice, e.g.\ for ontology completion or for plausible query answering. This will require two additional contributions. First, we either need a tractable fragment of the considered logics or an efficient approximate inference technique. Second, we need practical mechanisms to deal with the noisy nature of the available knowledge about betweenness (which typically would be learned from data) and the inconsistencies that may arise from applying interpolation (e.g.\ because concepts that were assumed to be natural may not be). To this end, we plan to study probabilistic or non-monotonic extensions of our framework.

\paragraph{Acknowledgments}
Yazm\'in  Ib\'a\~nez-Garc\'ia and Steven Schockaert have been supported by ERC Starting Grant 637277.

\bibliographystyle{kr}

\bibliography{references}

\clearpage
 \appendix
 \input{supplemental}

\end{document}

%% file: supplemental.tex
\section{\Large Appendix}

\section*{Proofs for Section~\ref{secFeatureSemantics}}

{\bf Lemma~\ref{prop:NatIncl} }{ \it
Let $\Imf= (\Imc, \Fmc, \pi)$ be a feature-enriched interpretation and  $D$  a natural concept in $\Imf$. Then  for every concept $C$,  $\varphi(D) \subseteq \varphi(C)$ iff $C^\Imc \subseteq D^\Imc$. }

\smallskip \noindent 
\begin{proof}
For the ``if" direction, we start by noting that for every $d \in C^\Imc$ it holds that $\pi(d) \supseteq \varphi(C) \supseteq \varphi(D)$. Then,  since $D$ is natural in $\Imf$ it holds that  $d \in \{e \in\Delta^\Imc \mid \varphi(D) \subseteq \pi(e)\} = D^\Imc$. Therefore, $C^\Imc \subseteq D^\Imc$.
The ``only if'' direction is trivial.  
\end{proof}

\smallskip \noindent{\bf Proposition~\ref{featureselem} }
{\it
Let $\Imf= (\Imc, \Fmc, \pi)$ be a feature-enriched interpretation. 
For every pair of natural concepts $C,D$ in $\Imf$ such that $C^\Imc \neq \emptyset$ and $D^\Imc\neq \emptyset$, it holds that $(C\bet D)^\Imc$ is equal to the following set:
$$\Bbf = \{d \in \Delta^\Imc \mid \exists d_1 \in C^\Imc. \exists  d_2 \in D^\Imc \text{ s.t. }  \mn{bet}(d_1, d, d_2)\}$$}

\smallskip \noindent 
\begin{proof} 
For every $d \in (C \bet D)^\Imc$,  by Condition 4 in Definition~\ref{defFeatureInterpretation}, there are $d_1, d_2 \in \Delta ^\Imc$ such that $\pi(d_1) = \pi(d) \cup \varphi(C)$ and $\pi(d_2)= \pi(d) \cup \varphi(D)$. Clearly, $\pi(d_1) \cap \pi(d_2) \subseteq \pi(d)$. Since $C$ and $D$ are both natural it holds that $d_1 \in C^\Imc$ and $d_2 \in D^\Imc$.  This shows $(C \bet D)^\Imc \subseteq \Bbf$. Now, let $d \in \Bbf$, and $d_1\in C^\Imc$ and  $d_2 \in D^\Imc$ be the elements witnessing this. We have that $\varphi(C) \subseteq \pi(d_1)$ and $\varphi(D) \subseteq \pi(d_2)$, and because $\pi(d_1) \cap \pi(d_2) \subseteq \pi(d)$ we can conclude $\varphi(C) \cap \varphi(D) \subseteq \pi(d)$. Therefore, $d \in (C \bet D)^\Imc$, showing that $\Bbf \subseteq (C\bet D)^\Imc$.  
\end{proof}

\smallskip \noindent 
{\bf Lemma~\ref{prop:ConjFeat}} {\it
Let $C,D$ be natural concepts w.r.t.\ a given TBox \Tmc. For every model $\Imf$ of $\Tmc$, it holds that $\varphi(C \sqcap D) = \varphi(C) \cup \varphi(D)$.}

\smallskip \noindent
\begin{proof}
Since  every $d \in (C \sqcap D)^\Imc$  is such that $\varphi(C) \subseteq \pi(d)$ and $\varphi(D) \subseteq \pi(d)$, we have that $\varphi(C) \cup \varphi(D) \subseteq \varphi(C \sqcap D)$. 

Now, to see that $\varphi(C \sqcap D) \subseteq \varphi(C) \cup \varphi(D)$, assume towards a contradiction that there is some $f \in \varphi(C\sqcap D)$ such that  $f \not \in \varphi(C)$, $f \not \in \varphi(D)$. Let $d \in (C \sqcap D)^\Imc$, we have that $f \in \pi(d)$. By Condition 4 in Definition~\ref{defFeatureInterpretation}, there is some $d' \in \Delta^\Imc$ such that $\pi(d')= \pi(d) \setminus \{f\}$. Moreover, $\varphi(C) \subseteq \pi(d')$ and $\varphi(D) \subseteq \pi(d')$. But since both $C$ and $D$ are natural, this means $d' \in C^\Imc \cap D^\Imc$, i.e.\ $d' \in (C\sqcap D)^\Imc$.  This yields a contradiction since it implies $f \not \in \varphi(C \sqcap D)$.  
\end{proof}
%

\smallskip \noindent
{\bf Theorem~\ref{lemma:intcj}}
{\it Let \Tmc be an $\EL^\bet$-TBox, 
and let $A,B,C,D$ be natural concepts w.r.t.\ \Tmc.
If $\Tmc \models \{ A \sqcap C \sqsubseteq B,  A\sqcap D \sqsubseteq B \}$
then $\Tmc \models  A \sqcap  (C \bet D)  \sqsubseteq B$.}

\smallskip \noindent
\begin{proof}  
By Lemma~\ref{lemma:int}, we have that $\Tmc \models (A \sqcap C) \bet (A \sqcap D) \sqsubseteq B$. Showing that $\Tmc \models A \sqcap (C \bet D) \sqsubseteq (A \sqcap C) \bet (A \sqcap D)$ yields the required conclusion. Let $\Imf =(\Imc, \Fmc, \pi)$ be a model of \Tmc.  Given that $C\bet D$ is natural in \Imf, the following equality holds by Lemma~\ref{prop:ConjFeat}:
\begin{align*}
\varphi(A \sqcap (C \bet D)) & = \varphi(A) \cup\varphi(C \bet D)\\
& = \varphi( A) \cup (\varphi(C)\cap\varphi(D))  \\
&=  (\varphi( A) \cup \varphi(C)) \cap (\varphi( A) \cup \varphi(D)) \\
& = \varphi((A \sqcap C) \bet (A \sqcap D))
\end{align*}
Finally, since  $(A\sqcap C) \bet (A\sqcap D)$ is also natural in $\Imf$, we can conclude using  Lemma~\ref{prop:NatIncl} that $(A \sqcap (C \bet D))^\Imc \subseteq ((A \sqcap C) \bet (A \sqcap D))^\Imc.$ 
%
\end{proof}

\section*{Proofs for Section~\ref{secGeometricSemantics}}

{\bf Proposition~\ref{propSoundnessGeoLeftNI}}
{\it Let \Tmc be an $\EL^\bet_{\textit{reg}}$-TBox.
If $\mathcal{T}\models \{A\lessdot(C,D), A \sqcap C \sqsubseteq B, D \sqsubseteq B\}$, with $A,B$ natural concepts and $C,D \in \mn{N^{Nat}_C}$, then $\mathcal{T} \models A\sqcap (C \bowtie D) \sqsubseteq B$. 
}

\smallskip\noindent
\begin{proof}
Let $\mathcal{I}$ be a geometric model of $\mathcal{T}$.
Let $p\in \text{reg}_{\mathcal{I}}(A\sqcap (C\bowtie D))$ and let us write $\kappa_{(C,D)}(p)= (p_1,p_2)$. Since $p\in \text{reg}_{\mathcal{I}}(A)$ and $\mathcal{I}\models A\lessdot(C,D)$, it follows that $p_1 \in \text{reg}_{\mathcal{I}}(A) \cap \text{reg}_{\mathcal{I}}(C)$, and thus, since $\mathcal{I}\models A \sqcap C \sqsubseteq B$, we also have $p_1\in \text{reg}_{\mathcal{I}}(B)$. Since $p_2\in \text{reg}_{\mathcal{I}}(D)$ and $\mathcal{I}\models D \sqsubseteq B$, we also have $p_2\in \text{reg}_{\mathcal{I}}(B)$. Since $p_1,p_2 \in \text{reg}_{\mathcal{I}}(B)$, $p$ is between $p_1$ and $p_2$, and $\text{reg}_{\mathcal{I}}(B)$ is convex, we can conclude $p\in \text{reg}_{\mathcal{I}}(B)$.
\end{proof}

\medskip\noindent{\bf Proposition~\ref{prop:ConjGeo}} {\it 
Let $A,B$ be natural concepts and $C,D$ natural concept names.
If $\mathcal{T} \models A\lessdot (C,D)$ and $\mathcal{T} \models B\lessdot (C,D)$ then $\mathcal{T} \models (A\sqcap B)\lessdot (C,D)$}

\smallskip\noindent
\begin{proof}
Let $\mathcal{I}$ be a geometric model of $\mathcal{T}$.
 Let $p\in \text{reg}_{\mathcal{I}}(A) \cap \text{reg}_{\mathcal{I}}(B) \cap \mn{conv}(\text{reg}_{\mathcal{I}}(C)\cup \text{reg}_{\mathcal{I}}(D))$, where we write $\kappa_{(A,B)}(p) = (p_1,p_2)$. If $\mathcal{I} \models \{A\lessdot (C,D),B\lessdot (C,D)\}$ then by definition we have $p_1 \in \text{reg}_{\mathcal{I}}(A) \cap \text{reg}_{\mathcal{I}}(C)$ and $p_1 \in \text{reg}_{\mathcal{I}}(B) \cap \text{reg}_{\mathcal{I}}(D)$, and thus also $p_1 \in \text{reg}_{\mathcal{I}}(A) \cap \text{reg}_{\mathcal{I}}(B) \cap \text{reg}_{\mathcal{I}}(C)$, from which we immediately find that $\mathcal{I}$ satisfies $(A\sqcap B)\lessdot (C,D)$.
\end{proof}

\section*{Proof of Theorem \ref{thm:upperfeatures}}

\begin{lemma}
Let $\alpha_1,...,\alpha_n,\beta$ be propositional clauses. Let $x$ be a fresh atom, then we have
\begin{align}\label{eqPropEntailment1}
\{\alpha_1,...,\alpha_n\} \models \beta
\end{align}
iff
\begin{align}\label{eqPropEntailment2}
\{\alpha_1\vee \neg x,...,\alpha_n \vee \neg x\} \models \beta \vee \neg x
\end{align}
\end{lemma}
\begin{proof}
Indeed, assume that \eqref{eqPropEntailment1} holds. Let $\omega$ be a model of $\{\alpha_1\vee \neg x,...,\alpha_n \vee \neg x\}$. We need to show that $\omega\models \beta\vee \neg x$. If $\omega \models \neg x$ this is trivial. However, if $\omega \models x$, then from $\omega \models \{\alpha_1\vee \neg x,...,\alpha_n \vee \neg x\}$ we obtain $\omega \models \{\alpha_1,...,\alpha_n\}$. Since we assumed \eqref{eqPropEntailment1} holds, this means $\omega\models \beta$, and thus in particular $\omega\models \beta \vee \neg x$.


Conversely, assume that \eqref{eqPropEntailment2} holds. Let $\omega'$ be a model of $\{\alpha_1,...,\alpha_n\}$ and let $\omega$ be the extension of $\omega'$ to $x$, such that $\omega\models x$. Since $\omega' \models \{\alpha_1,...,\alpha_n\}$ it follows that $\omega \models \{\alpha_1\vee \neg x,...,\alpha_n \vee \neg x\}$ and thus by the fact that we assumed \eqref{eqPropEntailment2}, we have $\omega \models \beta \vee \neg x$. Since $\omega(x)=T$ this means in particular $\omega\models \beta$ and thus also $\omega'\models \beta$.
\end{proof}

\noindent The significance of the above lemma is that to show \coNP-hardness, it is sufficient to show that we can reduce propositional entailment checking problems in which each clause has at least one negative literal. As will become clear in the following lemma, the reason why we need to restrict ourselves to such problems is related to the fact that $\bot$ does not appear in the language of $\EL^\bet$. Theorem \ref{thm:upperfeatures} follows immediately from the following lemma.

\begin{lemma}
Let $\alpha_1,...,\alpha_n,\beta$ be propositional clauses containing at least one negative literal. For such a clause $\alpha = x_1  \vee ... \vee x_n \vee \neg y_1 \vee ... \vee \neg y_m$ we define its translation $\tau(\alpha)$ as:
$$
X_1 \sqcap ... \sqcap X_n \sqsubseteq Y_1 \bowtie ...\bowtie Y_m
$$
where each atom $x$ occurring in $\{\alpha_1,...,\alpha_n,\beta\}$ is associated with a natural concept name $X$. In the particular case where there are no positive literals, the translation is given by:
$$
\top \subseteq Y_1 \bowtie ...\bowtie Y_m
$$
Then it holds that 
\begin{align}\label{eqPropEntailment3}
\{\alpha_1,...,\alpha_n\} \models \beta
\end{align}
iff 
\begin{align}\label{eqELBEntailment1}
\{\tau(\alpha_1),...,\tau(\alpha_n)\} \models \tau(\beta)
\end{align}
\end{lemma}
\begin{proof}
\begin{itemize}
\item First suppose that \eqref{eqPropEntailment3} holds. Let $\mathfrak{I} = (\mathcal{I},\mathcal{F},\pi)$ be a feature-enriched interpretation that satisfies $\{\tau(\alpha_1),...,\tau(\alpha_n)\}$. Note that because all concept names are natural, it holds that $\mathfrak{I} \models X_1 \sqcap ... \sqcap X_n \sqsubseteq Y_1 \bowtie ...\bowtie Y_m$ iff
$$
\varphi(Y_1 \bowtie ...\bowtie Y_m) \subseteq \varphi(X_1 \sqcap ... \sqcap X_n)
$$
which is equivalent to:
\begin{align}\label{eqConstraintFeaturesClause}
\varphi(Y_1) \cap ...\cap \varphi(Y_m) \subseteq \varphi(X_1) \cup ... \cup \varphi(X_n)
\end{align}
Now assume that there exists some interpretation $\mathfrak{I}$ such that $\mathfrak{I}\models \{\tau(\alpha_1),...,\allowbreak\tau(\alpha_n)\}$ but $\mathfrak{I}\not \models \tau(\beta)$. Let us write $\beta = a_1 \vee ... \vee a_s \vee \neg b_1 \vee ... \vee \neg b_t$. Then  $\mathfrak{I}\not \models \tau(\beta)$ means:
$$
\varphi(B_1) \cap ...\cap \varphi(B_t) \not\subseteq \varphi(A_1) \cup ... \cup \varphi(A_s)
$$
Now let $f$ be a feature that occurs in $\varphi(B_1) \cap ...\cap \varphi(B_t)$ but not in $\varphi(A_1) \cup ... \cup \varphi(A_s)$. Let us write $x_1,...,x_k$ for the set of all atoms that occur in $\{\alpha_1,...,\alpha_n,\beta\}$. Now we construct a propositional interpretation $\omega$ as follows:
\begin{align*}
\omega(x_i)= 
\begin{cases}
T & \text{if $f\in \varphi(X_i)$}\\
F & \text{otherwise}
\end{cases}
\end{align*}
First note that $\omega \models \{\alpha_1,...,\alpha_n\}$. Indeed, for a clause $\alpha_i$ of the form $x_1\vee ... \vee x_n \vee \neg y_1 \vee ... \vee \neg y_m$, since $\mathfrak{I}\models \tau(\alpha_i)$ we have that \eqref{eqConstraintFeaturesClause} holds, hence either we have $f \notin \varphi(Y_1) \cap ...\cap \varphi(Y_m)$ or we have $f\in \varphi(X_1) \cup ... \cup \varphi(X_n)$. We thus have:
\begin{align*}
&(f\notin \varphi(Y_1)) \vee ... \vee (f\notin \varphi(Y_m))\\
&\quad\quad\quad\vee (f\in \varphi(X_1)) \vee ... \vee (f\in \varphi(X_n))
\end{align*}
which is clearly equivalent to $\omega \models \alpha_i$. Since we assumed that \eqref{eqPropEntailment3} holds, the fact that we have $\omega \models \{\alpha_1,...,\alpha_n\}$ implies that we also have $\omega \models \beta$. But this is clearly not possible if $f \in (\varphi(B_1) \cap ...\cap \varphi(B_t)) \setminus (\varphi(A_1) \cup ... \cup \varphi(A_s))$, which is a contradiction. It thus follows that any model $\mathfrak{I}$ of $\{\tau(\alpha_1),...,\allowbreak\tau(\alpha_n)\}$ also has to be a model of $\tau(\beta)$.
\item Now conversely suppose that \eqref{eqELBEntailment1} holds. Let $\omega$ be a model of $\{\alpha_1,...,\alpha_n\}$, and let $x_1,...,x_k$ again be the atoms that occur in $\{\alpha_1,...,\alpha_n,\beta\}$. We now construct a feature enriched interpretation $\mathfrak = (\mathcal{I},\mathcal{F},\pi)$ as follows. We have $\Delta^{\mathcal{I}} = \{d_1,d_2,d_3\}$ and $\mathcal{F} = \{f_1,f_2\}$, where $\pi$ is defined as follows:
\begin{align*}
\pi(d_1) &= \emptyset&
\pi(d_2) &= \{f_1\}&
\pi(d_3) &= \{f_2\}
\end{align*}
Finally, $\mathcal{I}$ is defined as follows:
\begin{align*}
X_i^{\mathcal{I}} &= 
\begin{cases}
\{d_2\} & \text{if $\omega\models x_i$}\\
\{d_1,d_2,d_3\} & \text{otherwise}\\
\end{cases} 
\end{align*}
Note that $\varphi(X_i)=\{f_1\}$ if $\omega\models x_i$ and $\varphi(X_i)=\emptyset$ otherwise. It is straightforward to verify that $\omega\models \alpha_i$ implies $\mathfrak{I} \models \tau(\alpha_i)$. Since we assumed that \eqref{eqELBEntailment1} holds, this implies $\mathfrak{I}\models \tau(\beta)$. Finally, it is straightforward to verify that the latter implies $\omega \models \beta$. 
\end{itemize}
\end{proof}

\section*{Proof of Theorem \ref{thm:lowerfeatures}}
\begin{figure*}
    \centering
\begin{align*}
&\inference[\Rule{1}]{}{A \sqsubseteq A}
&&\inference[\Rule{2}]{}{A \sqsubseteq \top} \\[10pt]
& \inference[\Rule{3}]{A_1 \sqsubseteq A_2 \ \    A_2 \sqsubseteq A_3 }{A_1 \sqsubseteq A_3 } 
&& \inference[\Rule{4}]{A \sqsubseteq A_1 & A \sqsubseteq A_2 & A_1\sqcap A_2 \sqsubseteq C  }{ A \sqsubseteq C }\\
&\inference[\Rule{5}]{A \sqsubseteq B & B \sqsubseteq \exists r. C}{A \sqsubseteq\exists r. C}
&&\inference[\Rule{6}]{A \sqsubseteq \exists r. A_1 & A_1 \sqsubseteq B & \exists r. B_1 \sqsubseteq B}{A \sqsubseteq B}
\end{align*}
\caption{Saturation rules for \EL}
    \label{fig:rules}
\end{figure*}
{\bf Lemma~\ref{lema:boundedModel}} { \it 
Let \Tmc be an $\EL^\bet$ TBox. For every model $\Imf = (\Imc, \Fmc, \pi)$ of \Tmc, there is a model 
$\Jmf = (\Imc, \widehat \Fmc, \hat \pi)$ such that $|\widehat \Fmc | \leq \mn{poly}(\Tmc)$.  }

\medskip \noindent 
{\bf Proof Sketch.}
Let $\Imf = (\Imc, \Fmc, \pi)$ be a model of \Tmc with $\Imc = (\Delta^\Imc, \cdot^\Imc)$. We will construct a feature-based interpretation $\Jmf= (\widehat\Imc, \widehat\Fmc, \hat \pi)$, such that $\Delta^\Imc \subseteq \Delta^{\widehat \Imc}$ and 
$A^\Imc \subseteq  A^{\Imc'}$, for every concept $A$ name;
$r^{\Imc} \subseteq  r^{\widehat \Imc}$, for every role name; and $|\widehat \Fmc | \leq \mn{poly}(\Tmc)$.

\smallskip 
\noindent
Let $\mn{sub}_\Tmc$ be the set of concepts 
\begin{align*}
\mn{sub}_\Tmc := &  \{\top, \bot \} \cup \{A \mid \mn{N_C} \cap \mn{sig}(\Tmc)\}\\
&\quad \cup \{{A \bet B} \mid D \sqsubseteq A\bet B \in \Tmc, \text{ for some } D \}\\
&\quad  \cup \{A\sqcap B \mid A \sqcap B \sqsubseteq D \in \Tmc, \text{ for some } D\}
\end{align*} 
 Further, let $\equiv_{\Imc}$ be the following relation over the elements of  $\mn{sub}_\Tmc$. 
$$ \{(C,D) \mid C^\Imc  = D^\Imc \}.$$ 
Clearly, $\equiv_{\Imc}$  is an equivalence relation and thus induces a partition on $\mn{sub}_\Tmc$. 
Let $[C] := \{ D \mid C \equiv_\Imc D\}$ denote the \emph{equivalence class of $C$}.  For two distinct classes 
$[C]$, $[D]$, we write $[C] \prec [D]$ if there are $C' \in [C]$ and $D' \in [D]$ such that $C'^\Imc \subseteq D'^\Imc$. 
Let us denote with $\Kbf$ the set of all equivalnce classes. We will define inductively a mapping $\varphi': \Kbf \rightarrow 2^{\widehat\Fmc}$
where $\widehat \Fmc = \{f^1_X, f^2_X \mid X \in \mn{sub}_\Tmc  \setminus \{\top, \bot \}\}$. 
Let $\varphi_0$ be defined by 
\begin{align*}
\varphi_0([\top]) &= \emptyset;  \quad  \varphi_0([\bot])  = \widehat\Fmc; \\
\varphi_0([C]) &= \{ f^1_X, f^2_X \mid X \in [C] \} & \text{ if  neither }  \bot  \in [C] \text{ nor } \top \in [C] 
\end{align*}

Assuming that $\varphi_i$, for $i\geq 0$,  is defined, $\varphi_{i+1}$ is obtained by applying the following completion rules to every $[C] \in \Kbf$.  Let $F = \varphi_i([C] )$. 
\begin{itemize}
    \item $\varphi_{i+1}([C])$ is the result of adding to $F$ all the elements in $\varphi_i([D])$, for all $[D]$ such that $C \prec [D]$, as well as $\varphi_i([A]) \cap \varphi_i([B])$, for every concept of the form ${A\bet B}$ in  $[C]$.  
    \item For every $A \sqcap B \in [C]$,  with $A,B \in \mn{N^{Nat}_C}$, 
    \begin{align*}
    \varphi_{i+1}([A]) & = \varphi_{i}([A]) \cup \{f^1_X \in F \mid X \in \mn{sub}_\Tmc\}, \text{ and } \\
    \varphi_{i+1}([B]) &= \varphi_i([B]) \cup \{f^2_X \in  F\mid X \in \mn{sub}_\Tmc \}.   
   \end{align*}
\end{itemize}
Let $\hat\varphi$ be the mapping obtained in the limit. 

\smallskip
Towards defining $\Jmf$, let $\Delta$ be a set containing exactly one witness element $d_{F'}$ for each proper subset $F' \subset \widehat \Fmc$, and 
let $\hat \pi: \Delta^\Imc \cup \Delta \rightarrow \widehat \Fmc$ be defined by $\hat\pi(d_F)= F$, for every $d_F \in \Delta$ and 
$\hat\pi(d)= \bigcup_{d \in C^\Imc} \widehat\varphi([C])$, for every $d \in \Delta^\Imc$.  

We are ready to define the interpretation $\Imc'$. 
Then  $\Imc' = (\Delta^\Imc, \cdot^\Imc)$ is the interpretation with $\Delta^{\Imc'} = \Delta^\Imc \cup \Delta$ and
\begin{itemize}
\item $A^{\Imc'} = A^\Imc \cup \{ d_{F} \in \Delta \mid  \widehat\varphi(A) \subseteq F \}$,  for every concept name $A$;
\item $r^{\Imc'} = r^{\Imc} \cup \{ (d_{F}, d_{G}) \in \Delta \times \Delta \mid  \widehat\varphi(A) \subseteq  F \land \widehat\varphi(B) \subseteq G \land A \sqsubseteq \exists r. B \in \Tmc \}$. 
\end{itemize}

It is obvious that $| \widehat{\Fmc} | \leq \mn{poly}(\Tmc)$. 
It is routine to show that $\Jmf= (\Imc', \widehat\Fmc, \hat{\pi})$ satisfies Conditions 1--5 of Definition~\ref{defFeatureInterpretation}, and that it is a model of \Tmc. 


\medskip
\par\noindent{\bf Lemma~\ref{lemma:coNPUpper}} \ 
Let \Tmc be an $\EL^\bet$ TBox and $A,B$ concept names. Then, $\Tmc \not \models A \sqsubseteq B$ iff, after applying Steps 1-4 above, 
 $A \sqsubseteq B \not \in \Tmc'_\theta$. 

\medskip
\noindent
\begin{proof}
  ``$\Rightarrow$" direction.  
 Note that for every TBox $\Tmc$ and  every model $\Imf= (\Imc,\Fmc, \pi)$ of $\Tmc$, the mapping $\varphi$ defined by $\Imc$ and $\pi$ is proper for $\Tmc$. Further,  
 all CIs added to $\Tmc$ in Step 2 are consequence of the feature-based semantics according to Lemma~\ref{prop:NatIncl}; and the saturation rules (see Figure~\ref{fig:rules}) for  \EL  preserve properness of $\varphi$. Finally, that $A \sqsubseteq B \not \in \Tmc'_\varphi$ follows from soundness of the completion algorithm for \EL.  
 
  \newtheorem{claim}{Claim}
 \smallskip
 \noindent
 For the $\Leftarrow$ direction.  Our aim is to construct a canonical model $\Imf = (\Imc, \Fmc, \pi)$. 
 We will use $\theta$ and $\Tmc'_\theta$ to guide the construction. 
 We can assume that $\theta$ is a total mapping from the set of concept names in $\Tmc$ to some set of features $\Fmc$ (of polynomial size) such that $\theta(A)\subset \Fmc$, for every concept name $A$. 
 %

\smallskip 
We start by defining  $\Delta^\Imc = \{d_{A} \mid A\in \mn{N_C} \text{ occurs in } \Tmc'_\theta\}\cup \{d_X \mid X \subset \Fmc \}$ and set for every concept name $A$, 
$$\begin{aligned}
A^\Imc = &\{ d_{A'} \mid A' \sqsubseteq A  \in \Tmc'_\theta \} \, \cup \\
& \{d_X \mid \exists A'\in \mn{N^{Nat}_C},
\text{ s.t. } A' \sqsubseteq A \in \Tmc'_\theta, \theta(A') \subseteq X \}\end{aligned}$$ 
and for every role name $r$, $$r^\Imc = \{(d_A, d_{B}), (d_X, d_B) \mid A \sqsubseteq \exists r. B \text{ and } \theta(A)\subseteq X\}.$$
Further, set $\pi(d_A)= \theta(A)$ and $\pi(d_X)= X$. 

\medskip\noindent
The following is a consequence of $\theta$ being proper. 
\begin{claim}
    For every concept name $A$, 
$\varphi(A) = \theta(A)$. 
\end{claim}

\smallskip\noindent
Note that according to the definition of $\Imf$, we have that the following holds. 
\begin{itemize}
\item 
$d_{A'} \in A^\Imc$ iff $A'\sqsubseteq A\in\Tmc'_\theta$;
\item $d_{X} \in A^\Imc$ iff there is a natural concept name $A'$ (possibly equal to $A$)  whose defining features contain $X$. 
\end{itemize}

\medskip\noindent
We now show that $\Imf$ is indeed a model of $\Tmc$. 
 
\begin{itemize}
\item  For CIs of the form $A \sqsubseteq B \in \Tmc'_\theta$ we have that  for each $d_{A'} \in A^\Imc$ we have  
 $A' \sqsubseteq  A \in \Tmc'_\theta$. Saturation of $\Tmc'_\theta$ yields $A' \sqsubseteq B$ and therefore $d_{A'}\in B^\Imc$. Now, for each $d_X \in A^\Imc$, we know there is a natural concept name $A'$ such that $A' \sqsubseteq A \in \Tmc'_\theta$. Since $\theta$ is proper, we obtain the following $\pi(d_X) = X \supset \theta(A')\supseteq \theta(B')$. The definition of $\Imc$ yields that $d_X  \in B^\Imc$.  
 
\item  The argument for CIs of the form $A_1\sqcap  A_1 \sqsubseteq B$ in $\Tmc'_\theta$ is similar to the previous one, using the semantics of $\sqcap$. 
 \item For CIs of the form $A \sqsubseteq \exists r. B$ in $\Tmc'_\theta$. 
 Let $d \in A^\Imc$, then there is some $A'$ such that $A' \sqsubseteq A \in \Tmc'_\theta$, and either  $d= d_{A'}$, or $d= d_X$ and $A'$ is a natural concept name with $\theta(A') \subseteq X$.
 Because $\Tmc'_\theta$ is saturated by the rules, $A' \sqsubseteq \exists r. B$ is also in $\Tmc'_\theta$. The definition of $r^\Imc$, then ensures that  $(d_{A'}, d_{B}) \in r^\Imc$. Clearly, $d_B \in B^\Imc$ since $B \sqsubseteq B \in \Tmc'_\theta$. 
 
\item For CIs of the form $\exists r. A \sqsubseteq B$. 
    Let $d,d_{A'} \in \Delta^\Imc$ such that $(d,d_{A'}) \in r^\Imc$ and $d_{A'} \in A^\Imc$. From the definition of $r
    ^\Imc$, it then follows that there is some $B'$ such that $B' \sqsubseteq \exists r. A'$ and $A' \sqsubseteq A$ are both in $\Tmc'_\theta$ and $\theta(B') \subseteq X$, and either (i)~$d= d_{B'}$, or (ii)~$d= d_X$, and $B'$ is a natural concept name.  
    The saturation ensures $B' \sqsubseteq B \in \Tmc'_\theta$. Therefore in both cases   (i) and (ii) $d= d \in B^\Imc$. 
     
\item Consider now, CIs of the form $A \sqsubseteq B_1 \bet B_2$ and  let $d \in A^\Imc$.  It suffices to show that $\varphi(B_1) \cap \varphi(B_2) \subseteq \pi(d)$. By Claim 1, we have that $\varphi(B_1)\cap \varphi(B_2)= \theta(B_1)\cap \theta(B_2)$. Then, the definition of $\Imf$ implies that there is some $A'$ with $A' \sqsubseteq A \in \Tmc'_\theta$.  Since $\theta$ is proper we have that $\theta(B_1) \cap \theta(B_2) \subseteq \theta(A')$.  If $d= d_{A'}$, then clearly $\pi(d)$ contains $\varphi(B_1) \cap \varphi(B_2)$. On the other hand, if $d= d_X$, because $\varphi(A') \subseteq X$, we get as required  $\varphi(B_1) \cap \varphi(B_2) \subseteq \varphi(d)$. 

\item For CIs $B_1 \bet B_2 \sqsubseteq A$. Let $d \in (B_1\bet B_2)^\Imc$. Then $\varphi(B_1) \cap \varphi(B_2) \subseteq \pi(d)$. We consider two cases according to the form of $d$. 

\begin{itemize}

\item First, assume $d= d_A'$, for some $A'$. Then we have that $\pi(d)= \theta(A') \supset \theta(B_1) \cap \theta(B_2)= \varphi(B_1\bet B_2)$. The construction of $\Tmc_\theta$, then implies that $A' \sqsubseteq B_1 \bet B_2 \in \Tmc'_\theta$. Since in-between concepts are treated as concept names, this means saturation (rule \Rule{3}) ensures that $A' \sqsubseteq A\in \Tmc'_\theta$, and therefore $d \in A^\Imc$ in this case. 

\item Second, assume $d= d_X$ for some $X \subset \Fmc$.  Then since $\theta$ 
 is proper, we have that $\theta(A) \subseteq \theta(B_1) \cap \theta(B_2)= \varphi(B_1) \cap \varphi(B_2) \subseteq \pi(d)= X$. The definition of $\Imc$ then ensures $d \in A^\Imc$.   
\end{itemize}
\end{itemize}
 
Notably, the addition of elements of the form $d_X$ in $\Delta^\Imc$ ensures that every natural concept name is interpreted as a natural concept, and that Condition 4 of Definition~\ref{defFeatureInterpretation} is satisfied.

Finally, assume that $A \sqsubseteq B \not\in \Tmc'_\theta$. Then, by rule \Rule{1}
$A \sqsubseteq A \in \Tmc'_\theta$,  which means $d_A \in A^\Imc$, and $d_A \not \in B^\Imc$. Since $\Imf$ is also a model of $\Tmc$ this yields $\Tmc \not \models A\sqsubseteq B$. 
\end{proof}


\section*{Proof of Theorem \ref{thm:lowergeo}}
Let us consider a generalized CP-net over the set of atoms $\mathcal{A}=\{a_1,...,a_m\}$, with rules $\rho_1,...,\rho_n$ of the form \eqref{eqCPrule}. Let $(l_1,...,l_m)$ be the initial outcome, and let the $\EL^\bet$ TBox $\mathcal{T}$ be the translation of the generalized CP-net, as defined in Section \ref{secHardnessGeometric}.

One direction of the proof of Theorem \eqref{thm:lowergeo} is straightforward to show, as expressed in the following lemma.
\begin{lemma}
If there exists a sequence of improving flips from the initial outcome $(l_1,...,l_m)$ to another outcome $(r_1,...,r_m)$ then it follows that
$$
\mathcal{T} \models \tau(r_1 \wedge ... \wedge r_m) \sqsubseteq Z
$$
\end{lemma}
\begin{proof}
This follows trivially from the soundness of interpolation, given that each improving flip can be simulated by applying interpolation.
\end{proof}

\noindent We focus on showing the opposite direction. In particular, for an outcome $(r_1,...,r_m)$ such that $\mathcal{T} \models \tau(r_1 \wedge ... \wedge r_m) \sqsubseteq Z$, we show that there must exist a sequence of improving flips from $(l_1,...,l_m)$ to $(r_1,...,r_m)$.

To this end, it is sufficient to show that there exists some particular geometric model $\mathcal{I}$ of $\mathcal{T}$ such that $\mathcal{I} \models \tau(r_1 \wedge ... \wedge r_m) \sqsubseteq Z$ iff there exists a sequence of improving flips from $(l_1,...,l_m)$ to $(r_1,...,r_m)$.
We construct such a model $\mathcal{I}$ as follows. Let $\omega_1,...,\omega_n$ be an enumeration of all the possible outcomes (or, equivalently, an enumeration of the possible worlds over $\mathcal{A}$). For $a_i\in \mathcal{A}$, we define $\text{reg}_{\mathcal{I}} (A_i)$ as the set of all points $(x_1,...,x_n)\in \mathbb{R}^n$ such that:
\begin{itemize}
\item $(\omega_i \models a_i) \Rightarrow (0 \leq x_{i} \leq 1)$,
\item $(\omega_i \models \neg a_i) \Rightarrow (x_{i}=0)$,
\item $\sum_i x_i = 1$.   
\end{itemize}
Similarly, we define $\text{reg}_{\mathcal{I}} (\overline{A_i})$ as the set of all points $(x_1,...,x_n)$ such that:
\begin{itemize}
\item $(\omega_i \models \neg a_i) \Rightarrow (0 \leq x_{i} \leq 1)$,
\item $(\omega_i \models  a_i) \Rightarrow (x_{i}=0)$,
\item $\sum_i x_i = 1$.   
\end{itemize}
For an outcome $\omega=(r_1,...,r_m)$ it holds that $\text{reg}_{\mathcal{I}}(\tau(r_1)) \cap ... \cap \text{reg}_{\mathcal{I}}(\tau(r_m))$ contains a single point, which we will denote as $P(\omega)$. Note that the coordinate of $P(\omega)$ corresponding to the outcome $\omega$ is 1 and all other coordinates are 0.

The natural concept names $W_i$ were introduced as an abbreviation of $\eta_i \wedge \neg q_i$. Accordingly, we define:
$$
\text{reg}_{\mathcal{I}} (W_i) = \text{reg}_{\mathcal{I}} (\tau(\eta_i \wedge \neg q_i))
$$
For a CP-rule $\rho_i$ let us write $\mathcal{O}_i$ for the set of all improving flips $(\omega,\omega')$ that are warranted by this rule. With each such an improving flip, we associate the point $p_{(\omega,\omega')} = (x_1,...,x_n)$ where:
\begin{align*}
x_i = 
\begin{cases}
-1 & \text{if $\omega_i = \omega$}\\
2 & \text{if $\omega_i = \omega'$}\\
0 & \text{otherwise}
\end{cases}
\end{align*}
We define:
$$
\text{reg}_{\mathcal{I}} (X_i) = \mn{conv}(\bigcup_{(\omega,\omega')\in \mathcal{O}_i} p_{(\omega,\omega')})
$$
and
\begin{align*}
\text{reg}_{\mathcal{I}} (Z) &= \mn{conv}\left( \{P(l_1,..., l_m)\} \cup \bigcup_i \text{reg}_{\mathcal{I}} (X_i) \right)
\end{align*}

\begin{lemma}
The $n$-dimensional geometric interpretation $\mathcal{I}$ defined above is a model of $\mathcal{T}$.
\end{lemma}
\begin{proof}
We have that \eqref{eqInitialELrule}, \eqref{eqELrule2}, \eqref{eqELruleWi1} and \eqref{eqELruleWi2} are trivially satisfied. 
To see why \eqref{eqELrule3} is satisfied, consider any outcome $\omega' = (s_1,...,s_m)$ such that $\omega'\models \eta_i \wedge  q_i$. Let $\omega$ be the outcome which coincides with $\omega'$ apart from the fact that $\omega$ satisfies $\neg q_i$ instead of $q_i$. Note that $(\omega,\omega') \in \mathcal{O}_i$ and thus $p_{(\omega,\omega')} \in \text{reg}_{\mathcal{I}} (X_i)$. Moreover, we also have $P(\omega) \in \tau(\eta_i \wedge \neg q_i)$, and we have that:
$$
P(\omega') = \frac{1}{2} P(\omega) + \frac{1}{2} p_{(\omega,\omega')}
$$
We thus find $P(\omega')\in \text{reg}_{\mathcal{I}}( \tau(\eta_i\wedge \neg q_i)\bowtie X_i)$. Finally, note that each point in $\text{reg}_{\mathcal{I}}(\tau(\eta_i\wedge q_i))$ can be written as a convex combination of points of the form $P(\omega')$, hence we have 
$$
\text{reg}_{\mathcal{I}}(\tau(\eta_i \wedge q_i)) \subseteq \text{reg}_{\mathcal{I}}( \tau(\eta_i\wedge \neg q_i)\bowtie X_i)
$$
Given that $\text{reg}_{\mathcal{I}}( \tau(\eta_i\wedge\neg q_i))=\text{reg}_{\mathcal{I}}( W_i)$, we thus have that $\mathcal{I}$ satisfies \eqref{eqELrule3}. Finally, we show that the non-interference assertions are all satisfied. For each rule $\rho_i$, we need to define a suitable mapping $\kappa_{(W_i,X_i)}$. For the ease of presentation we will denote this mapping as $\kappa_i$. Let $p \in \text{reg}_{\mathcal{I}} (W_i\bowtie X_i)$, which is equivalent to $p \in \text{reg}_{\mathcal{I}}( \tau(\eta_i \wedge \neg q_i)\bowtie X_i)$. Then it holds that
\begin{align}\label{eqConvexCombinationA1}
p = \lambda_0 p_1 + \sum_{(\omega,\omega')\in\mathcal{O}_i} \lambda_{(\omega,\omega')} p_{(\omega_{j_l},\omega_{j_l}')}
\end{align}
for some $p_1 \in \text{reg}_{\mathcal{I}}( \tau(\eta_i\wedge\neg q_i))$, where $\lambda_0$ and the weights of the form $\lambda_{(\omega,\omega')}$ are non-negative and sum to 1. Let $x_{j_1},...,x_{j_d}$ be the non-zero coordinates from $p=(x_1,...,x_n)$ that correspond to outcomes $\omega'_{j_1},...,\omega'_{j_d}$ in which $\eta_i\wedge q_i$ is true. Note that for each of these outcomes $\omega'_{j_l}$ there is a unique improving flip $(\omega_{j_l},\omega'_{j_l})$ in $\mathcal{O}_i$. Moreover, for $(\omega,\omega')\in \mathcal{O}_i \setminus \{(\omega_{j_1},\omega_{j_1}'),...,(\omega_{j_d},\omega_{j_d}')\}$ we have that $\lambda_{(\omega,\omega')}=0$. Indeed, for such a flip $(\omega,\omega')$ we have that $x_{\omega'}=0$, whereas the coordinate corresponding to $\omega'$ in $p_{(\omega,\omega')}$ is strictly positive, and the coordinate corresponding to $\omega'$ in $p_{(\mu,\mu')}$ for any other improving flip $(\mu,\mu')\in \mathcal{O}_i$ is 0. Hence we can rewrite \eqref{eqConvexCombinationA1} as:
$$
p = \lambda_0 p_1 + \sum_{l=1}^d \lambda_l p_{(\omega_{j_l},\omega_{j_l}')}
$$
where $\sum_{l=0}^d \lambda_l= 1$.
If $\lambda_0 >0$, the point $p_1$ is uniquely defined. If moreover $\lambda_0<1$,  we can define $\kappa_i(p)=(p_1,p_2)$ with
$$
p_2 = \sum_{l=1}^s \frac{\lambda_l}{1-\lambda_0} p_{(\omega_{j_l},\omega_{j_l}')}
$$
where we indeed have $p_2\in \text{reg}_{\mathcal{I}}(X_i)$. Moreover, note that $p\in \text{reg}_{\mathcal{I}}(A_j)$ iff $p_1\in \text{reg}_{\mathcal{I}}(A_j)$ and $p\in \text{reg}_{\mathcal{I}}(\overline{A_j})$ iff $p_1\in \text{reg}_{\mathcal{I}}(\overline{A_j})$ for all atoms $a_j$ that do not appear in $\eta_i$ and $q_i$, i.e.\ the required non-interference assertions are all satisfied. In the case where $\lambda_0=0$, we can choose $p_2=p$ and choose $p_1$ arbitrarily from $\text{reg}_{\mathcal{I}}( \tau(\eta_i \wedge \neg q_i))$. Indeed, we then have $p\notin \text{reg}_{\mathcal{I}}(A_j)$ and $p\notin \text{reg}_{\mathcal{I}}(\overline{A_j})$ for all atoms $a_j$, hence the non-interference assertions are trivially satisfied. Finally, if $\lambda_0=1$, we can choose $p_1=p$ and choose $p_2$ arbitrarily from $\text{reg}_{\mathcal{I}}(X_i)$, in which case the non-interference assertions are also trivially satisfied. 
\end{proof}

\begin{lemma} It holds that $\mathcal{I} \models \tau(r_1\wedge ... \wedge r_m) \sqsubseteq Z$ iff there exists a sequence of improving flips from $(l_1,...,l_m)$ to $(r_1,...,r_m)$.
\end{lemma}
\begin{proof}
Let us write $\Omega$ for the set of all outcomes $(s_1,..., s_m)$ that can be obtained from $(l_1,...,l_m)$ using a sequence of improving flips, where we will assume that $\Omega$ also includes $(l_1,...,l_m)$ itself. Let us write $\Psi$ for the set of all outcomes $(s_1,...,s_m)$ for which
$$
P(s_1,... ,s_m) \in \text{reg}_{\mathcal{I}} (Z)
$$
We will show that $\Omega=\Psi$, from which the stated result immediately follows. 

Note that if $\rho_i$ sanctions the flip from some outcome $\omega_t = (t_1,...,t_m)$ to some outcome $\omega_s =(s_1,...,s_m)$ then it holds that 
$$
P(\omega_s) = \frac{1}{2}P(\omega_t) + \frac{1}{2}p_{(\omega_s,\omega_t)}
$$
and thus we have in particular
$$
P(\omega_s)  \in \mn{conv}(\{P(\omega_t)\} \cup \text{reg}_{\mathcal{I}} (X_i))
$$
From this, we find that if there exists a sequence of improving flips from $\omega_l = (l_1,...,l_m)$ to $\omega_s$, it must be the case that $P(\omega_s)\in \text{reg}_{\mathcal{I}} (Z)$. We thus already have that $\Omega \subseteq \Psi$.

Conversely, let $\omega_s \in \Psi$. Then $P(\omega_s) \in \mn{conv}( \{P(\omega_l)\} \cup \bigcup_i \text{reg}_{\mathcal{I}} (X_i))$. Hence, there must exist improving flips $(\omega_{u_1},\omega_{v_1}),...,(\omega_{u_k},\omega_{v_k})$ and weights $\lambda_0,\lambda_1,...,\allowbreak\lambda_k \geq 0$ such that $\sum_{i=0}^k \lambda_i =1$ and
$$
P(\omega_s) = \lambda_0 P(\omega_l) + \sum_{i=1}^k \lambda_i p_{(\omega_{u_i},\omega_{v_i})}
$$
For each $i\in \{1,...,k\}$ such that $\lambda_i>0$, one of the following needs to hold:
\begin{itemize}
    \item $\omega_{u_i} = \omega_l$;
    \item $\omega_{v_i} = \omega_s$;
    \item there are $j,l\in \{1,...,k\}$ such that $\lambda_j,\lambda_l>0$, $\omega_{u_i}=\omega_{v_j}$ and $\omega_{v_i}=\omega_{u_l}$.
\end{itemize}
Indeed, suppose there were some $i$ such that $\omega_{u_i} \neq \omega_l$, $\omega_{v_i} \neq \omega_s$, $\omega_{u_i}\neq \omega_{v_j}$, for all $j\in \{1,...,k\}$ such that $\lambda_j>0$; the case where no $l$ with $\lambda_l>0$ exists for which $\omega_{v_i}=\omega_{u_l}$ is entirely analogous. The coordinate of $p_{(\omega_{u_i},\omega_{v_i})}$ that corresponds to the outcome $\omega_{u_i}$ is strictly negative, whereas there are no points among $\{P(\omega_l),p_{(\omega_{u_1},\omega_{v_1})},...,p_{(\omega_{u_k},\omega_{v_k})}\}$ for which this coordinate is strictly positive. This would mean that the coordinate corresponding to $\omega_{u_i}$ is also negative in $P(\omega_s)$, which is a contradiction.

Since we assumed that the given generalized CP-net is consistent, there cannot be cycles of improving flips. This means that $(\omega_{u_1},\omega_{v_1}),...,\allowbreak (\omega_{u_k},\omega_{v_k})$ define one or more sequences of improving flips from $\omega_l$ to $\omega_s$.
\end{proof}